\documentclass[sigconf]{conference}

\usepackage{balance}

\acmSubmissionID{fp0566}

\title[Efficient and Optimal Policy Gradient Algorithm for Corrupted Multi-armed Bandits]{Efficient and Optimal Policy Gradient Algorithm for Corrupted Multi-armed Bandits}

\author{Jiayuan Liu}
\affiliation{
  \institution{Carnegie Mellon University}
  \city{Pittsburgh}
  \country{United States}}
\email{jiayuan4@andrew.cmu.edu}

\author{Siwei Wang}
\affiliation{
  \institution{Microsoft Research Asia}
  \city{Beijing}
  \country{China}}
\email{siweiwang@microsoft.com}

\author{Zhixuan Fang}\thanks{Corresponding authors: Siwei Wang (\texttt{siweiwang@microsoft.com}), Zhixuan Fang (\texttt{zfang@mail.tsinghua.edu.cn}).}
\affiliation{
  \institution{Tsinghua University}
  \institution{Shanghai Qi Zhi Institute}
   \country{China}}
\email{zfang@mail.tsinghua.edu.cn}

\begin{abstract}
In this paper, we consider the stochastic multi-armed bandits problem with adversarial corruptions, where the random rewards of the arms are partially modified by an adversary to fool the algorithm. We apply the policy gradient algorithm SAMBA to this setting, and show that it is computationally efficient, and achieves a state-of-the-art $O(K\log T/\Delta) + O(C/\Delta)$ regret upper bound, where $K$ is the number of arms, $C$ is the unknown corruption level, $\Delta$ is the minimum expected reward gap between the best arm and other ones, and $T$ is the time horizon. Compared with the best existing efficient algorithm (e.g., CBARBAR), whose regret upper bound is $O(K\log^2 T/\Delta) + O(C)$, we show that SAMBA reduces one $\log T$ factor in the regret bound, while maintaining the corruption-dependent term to be linear with $C$. This is indeed asymptotically optimal. We also conduct simulations to demonstrate the effectiveness of SAMBA, and the results show that SAMBA outperforms existing baselines.
\end{abstract}

\keywords{Multi-armed Bandits; Corruption; Policy Gradient}

\usepackage{times}
\usepackage{algorithm}
\usepackage{algpseudocode}

\usepackage[utf8]{inputenc} % allow utf-8 input
\usepackage[T1]{fontenc}    % use 8-bit T1 fonts
\usepackage{hyperref}       % hyperlinks
\usepackage{url}            % simple URL typesetting
\usepackage{booktabs}       % professional-quality tables
\usepackage{amsfonts}       % blackboard math symbols
\usepackage{nicefrac}       % compact symbols for 1/2, etc.
\usepackage{microtype}      % microtypography
\usepackage{xcolor}         % colors

\usepackage{appendix}
\usepackage{verbatim}
\usepackage{subcaption}
\usepackage{fancyhdr}
\usepackage{booktabs}
\usepackage{amsmath,amsfonts,amsthm} % amssymb
\usepackage{fancyhdr}
\usepackage{extramarks}
\usepackage{chngpage}
\usepackage{soul,color}
\usepackage{graphicx,float}
\usepackage{diagbox}

\usepackage{stfloats}

\usepackage{bm}

\newtheorem{theorem}{\bf{Theorem}}
\newtheorem{remark}{\bf{Remark}}

\newtheorem{lemma}[theorem]{\bf{Lemma}}

\newtheorem{definition}[theorem]{\bf{Definition}}

\newtheorem{fact}[theorem]{\bf{Fact}}

\newcommand{\E}{\mathbb{E}}

\newcommand{\argmax}{\operatornamewithlimits{argmax}}

\newcommand{\BibTeX}{\rm B\kern-.05em{\sc i\kern-.025em b}\kern-.08em\TeX}

\begin{document}

%%% The following commands remove the headers in your paper. For final 
%%% papers, these will be inserted during the pagination process.

\pagestyle{fancy}
\fancyhead{}

%%% The next command prints the information defined in the preamble.

\maketitle 

%%%%%%%%%%%%%%%%%%%%%%%%%%%%%%%%%%%%%%%%%%%%%%%%%%%%%%%%%%%%%%%%%%%%%%%%

\section{Introduction}

Multi-armed bandits (MAB) model requires the learning policy to learn from feedback to optimize decision-making in complex and uncertain environments~\cite{gupta2019better}. 
In this model, there are $K$ arms, and each arm $a$ is associated with a reward distribution $\mathcal{F}_a$. In each round $t\in[T]$, a player can choose one arm $a$ from the $K$ arms to pull and observe a reward $R_a\sim \mathcal{F}_a$. 
Denote $r_a$ the expected reward of arm $a$, $a^* = \argmax_{a} r_a$ the optimal arm, and $r^* = \max_a r_a$ the highest expected reward.
Then we let $\Delta_a=r^*-r_a$, $\Delta=\min_{\Delta_a>0}\Delta_a$, 
and define cumulative regret as the expected difference between the cumulative reward from pulling the optimal arm and the cumulative reward of the algorithm, i.e., pulling arm $a$ once incurs a regret of $\Delta_a$.
The aim of the player is to choose arms properly to minimize cumulative regret.

MAB captures the basic tradeoff between exploration and exploitation in online learning, and is widely adopted in real-world applications, e.g., when a news website picks an arriving article header to show to maximize the users' clicks, and when an investor chooses a stock for investment to maximize the total wealth~\cite{slivkins2019introduction}. 
Because of this, there is abundant research related to MAB problems, which proposes solutions including 
Upper Confidence Bound~\cite{auer2002finite}, Active Arm Elimination~\cite{even2006action}, Thompson Sampling~\cite{agrawal2017near}, etc.

However, in some applications, such as a recommendation system that suggests restaurants to customers, while most inputs follow a stochastic pattern from a fixed distribution, some inputs would be corrupted, e.g., injected by fake reviews from the restaurant's competitors~\cite{lykouris2018stochastic}. 
In addition, in machine learning applications, data may be imperfect or manipulated. Studying corruption bandits helps develop learning algorithms that remain effective even when data is corrupted, which is useful in fields such as federated learning~\cite{duanyi2023constructing} and distributed sensor networks~\cite{chen2023dynamic}. 

Corruption also exists in other applications such as online advertising and cybersecurity. 
In this paper, we consider the stochastic multi-armed bandits problem with adversarial corruptions, where the rewards of the arms are partially modified by an adversary to fool the algorithm~\cite{he2022nearly,kapoor2019corruption}. 
At each time step $t$, before an arm is pulled, the adversary can make corruptions, i.e., shift the expected reward of any arm $a$ to any corrupted value with cost $\max_a |r_a - r_a'(t)|$, where $r_a'(t)$ is the expected reward of arm $a$ after such corruption. The only constraint for the adversary is that his total cost cannot exceed corruption level $C$, i.e., $\sum_t \max_a |r_a - r_a'(t)| \le C$, while this $C$ also keeps unknown to the player. 

Existing algorithms pay a high cost for robustness against adversarial corruptions. 
The current state-of-the-art \emph{combinatorial algorithms} (i.e., those containing solely combinatorial operations such as enumeration and basic calculations, and with computational cost in each time step independent of the time horizon $T$)
exhibit a regret upper bound of $O(\log^2 T + C)$ for corrupted bandits, e.g.,~\citet{xu2021simple}. 
This implies that the algorithm's regret is not tight (i.e., it has one more $\log T$ factor compared to the $\Omega(\log T)$ regret lower bound~\cite{lykouris2018stochastic}), and suffers an $O(\log^2 T)$ regret even if there is no corruption. 

In this paper, our aim is to solve the above challenge and find efficient bandit algorithms that can handle adversarial corruptions without overhead, i.e., the regret upper bound approaches the bound in standard MAB as corruption level $C$ decays to zero.
Recent work~\cite{denisov2020regret} proposes a combinatorial algorithm Stochastic Approximation Markov Bandit Algorithm (SAMBA) to solve the standard MAB problem. 
In this paper, we employ this algorithm to address the corrupted bandits problem.

We are interested in SAMBA due to its adoption of a Markovian policy, in which the distribution of the chosen arm at step $t+1$ depends solely on the distribution of the chosen arm at step $t$, as well as the chosen arm and the observation at step $t$.
This is a desired property for corrupted bandits and reduces the complexity of the analysis.
Based on such a property, we show that 
SAMBA achieves a regret upper bound of $O(\log T + C)$, which is a major improvement compared to the $O(\log^2 T+C)$ regret upper bound of the best existing combinatorial algorithms. Meanwhile, our regret upper bound matches i) the $\Omega(\log T)$ regret lower bound when there is no corruption; and ii) the $\Omega(C)$ regret lower bound with corruption level $C$. This shows that SAMBA is indeed asymptotically optimal. 
We also conduct experiments to compare the performance of SAMBA with other existing baselines, whose results demonstrate the empirical effectiveness of SAMBA. 

\subsection{Our Main Contribution}
The aim of this research is to develop a combinatorial anti-corruption multi-armed bandits algorithm that is fast, easily implementable, and has a better performance guarantee than existing works. 
We employ SAMBA algorithm to tackle the corrupted bandits problem, marking the inaugural utilization of a policy gradient algorithm in this scenario.

Our primary contribution lies in three aspects. Firstly, we are the first to employ and analyze combinatorial policy gradient algorithms in the context of corrupted bandits. 
Secondly, we theoretically prove SAMBA's exceptional performance in the corrupted bandits setting. In addition, we demonstrate the empirical performance advantage of SAMBA over existing baselines. 
Our analysis is groundbreaking, as it is the first to prove that a combinatorial algorithm can achieve the optimal regret upper bound in the corrupted bandits setting. This result highlights the significance of our work in advancing the understanding and practicality of combinatorial approaches for dealing with corruption in bandit problems.

\subsection{Related Work}

\begin{table*}[!ht] %\small
    \renewcommand\arraystretch{2}
    \centering
    \caption{Comparison of different corrupted bandits algorithms. }\label{Table_1} 
    \begin{tabular}{|p{5.5cm}|c|c|c|}
    \hline
    \textbf{Algorithm} &\textbf{Known} $C$ &\textbf{Combinatorial} & \textbf{Regret Bound}\\ \hline
     Fast-Slow AAE Race~\cite{lykouris2018stochastic} & Yes & Yes& $O\big(KC\sum_{i\neq i^*}\frac{\log^2 T}{\Delta_i}\big)$ \\ \hline%w.p. $1-\delta$\\ \hline
     Multi-Layer AAE Race~\cite{lykouris2018stochastic} & No & Yes& $O\big(KC\sum_{i\neq i^*}\frac{\log^2 T}{\Delta_i}\big)$\\ \hline
    BARBAR~\cite{gupta2019better} & No & Yes& $O\big(KC+\sum_{i\neq i^*}\frac{\log^2 T}{\Delta_i}\big)$\\ \hline
    Cooperative Bandit Algorithm Robust to Adversarial Corruptions~\cite{liu2021cooperative} & No & Yes& $O\big(C+\frac{K\log^2 T}{\Delta}\big)$\\ \hline 
    CBARBAR~\cite{xu2021simple}  & No & Yes& $O\big(C+\sum_{i\neq i^*}\frac{\log^2 T}{\Delta_{i}}\big)$\\ \hline
    Tsallis-INF~\cite{zimmert2021tsallis}, FTRL~\cite{jin2020simultaneously}, FTPL~\cite{honda2023follow} & No & No& $O\big(C+\frac{K\log T}{\Delta}\big)$\\ \hline 
    \textbf{SAMBA~\cite{denisov2020regret} (with our analysis)} & \textbf{No} & \textbf{Yes} & $\boldmath{O\big(\frac{C}{\Delta}+\frac{K\log T}{\Delta}\big)}$ \\ \hline
    Regret Lower Bound~\cite{lykouris2018stochastic}& -- & -- & $\Omega\big(C+\frac{K\log T}{\Delta}\big)$\\ \hline
    \end{tabular}
\end{table*}

\citet{lykouris2018stochastic} is the first to consider stochastic bandits with adversarial corruptions.
They propose Fast-Slow Active Arm Elimination Race algorithm that achieves a high probability regret upper bound of $O\big(KC\sum_{i\neq i^*}\frac{\log^2 T}{\Delta_i}\big)$ when $C$ is known, and Multi-layer Active Arm Elimination Race algorithm that achieves the same high probability regret upper bound when $C$ is unknown.
They also show that a linear degradation to the total corruption amount $C$ is the best one can do, i.e., with corruption level $C$, any algorithm must suffer a regret lower bounded by $\Omega(C)$. 

\citet{gupta2019better} introduces a new algorithm called 
BARBAR, which reduces the regret upper bound to $O\big(KC+\sum_{i\neq i^*}\frac{\log^2 T}{\Delta_i}\big)$ when $C$ is unknown. 
\citet{liu2021cooperative,xu2021simple} make some further improvements on BARBAR, providing the solutions under cooperative bandits setting~\cite{liu2021cooperative} and combinatorial bandits setting~\cite{xu2021simple}. In addition, they reduce the $O(KC)$ term in the regret upper bound to $O(C)$, i.e.,  the regret upper bound of~\citet{liu2021cooperative} is $O\big(C+\frac{K\log^2 T}{\Delta}\big)$, and the regret upper bound for~\citet{xu2021simple} is $O\big(C+\sum_{i\neq i^*}\frac{\log^2 T}{\Delta_{i}}\big)$.

Except for the combinatorial algorithms that come from traditional bandit literature, there is another type of non-combinatorial algorithms that come from ``best-of-both-worlds (BOBW)'' literature. In BOBW, the algorithm needs to ensure good regret performance under both the stochastic scenario %(no corruption) 
and the totally adversarial scenario 
\cite{bubeck2012best}. Some of the BOBW algorithms also perform well in corrupted bandits. For example, 
\citet{zimmert2021tsallis} uses Tsallis-INF algorithm with Tsallis entropy regularization and~\citet{jin2020simultaneously} uses Follow-the-Regularized-Leader (FTRL) method~\cite{audibert2009minimax,zimin2013online,zimmert2019connections} with a novel hybrid regularizer to solve the corrupted bandits problem, both of which are based on online mirror descent (OMD) method and lead to a regret upper bound of $O\big(C+\frac{K\log T}{\Delta}\big)$\footnote{Though they claimed that their regret upper bound is $O\big(\frac{K\log T}{\Delta}+\sqrt{\frac{CK\log T}{\Delta}}\big)$, we emphasize that the definition of their regret is not the same as the one we and other corrupted bandits works use. In fact, there is an $O(C)$ gap between the two kinds of regret.}. 
However, OMD algorithms are not combinatorial and require more computational power than combinatorial algorithms such as BARBAR and SAMBA. 
Specifically, in each time step, the OMD algorithms need to solve a convex optimization problem, and the regret analysis is based on the actions corresponding to the optimal points. 
In practice, one can only use optimization algorithms (e.g., gradient descent) to look for near-optimal points. 
Since there are totally $T$ convex optimization problems to solve, to guarantee a similar regret bound, the gap between the approximate points and the optimal points should depend on $T$ (e.g., ${1\over T}$).
Therefore, the complexity required for each step also depends on $T$. 
As a comparison, combinatorial algorithms (e.g., BARBAR and SAMBA) only need $O(K)$ additions or multiplications in each time step.

Recent findings show that sampling algorithms can be more computationally efficient than optimization algorithms~\cite{ma2019sampling,sun2023revisiting}. 
% \jiayuan{Added FTPL. }
\citet{honda2023follow} incorporates this idea and uses follow-the-perturbed-leader-based (FTPL) method~\cite{abernethy2015fighting,kalai2005efficient} which replaces the procedure of solving the optimization problem in FTRL by multiple samplings and resamplings. 
However, FTPL does not completely solve the complexity challenge. Specifically, though the expected computation cost at each time step is $O(K)$, the variance of computation cost at each time step is $O(T)$, making it still non-combinatorial.

An overall comparison of different algorithms is given in Table \ref{Table_1}. Note that the state-of-the-art combinatorial algorithms have $O(\log^2 T+C)$ regret upper bounds, 
while only non-combinatorial algorithms can achieve $O(\log T + C)$ regret upper bound. 
Our analysis shows that a combinatorial algorithm, SAMBA, can achieve an $O(\log T + C)$ regret upper bound, which matches the regret lower bound for the corrupted bandits.

\section{Preliminaries}

\subsection{Multi-armed Bandits}

A multi-armed bandits instance is a tuple $(\mathcal{A}, \bm{r}, T)$. Here, i) $\mathcal{A} = \{1,2,\cdots,K\}$ is the set of arms and $K$ is the number of arms; ii) $\bm{r} = [r_1,\cdots,r_K] \in [0,1]^K$ are the corresponding expected rewards of the $K$ arms; and iii) $T$ is the time horizon. 
At each time step $t \le T$, the player must choose an arm $a(t) \in \mathcal{A}$ to pull. After that, he will receive a random reward (feedback) $R_{a(t)}(t)$.
In this paper, for simplicity of analysis, we focus on the case that the random rewards are Bernoulli, i.e., $R_{a(t)}(t)$ are drawn from a Bernoulli distribution with mean $r_{a(t)}$ independently. Our results can be easily extended to the general bounded-reward case.

The player can use the history information $\mathcal{H}_{t-1}$ to generate a random distribution $p(t)$ on the action set $\mathcal{A}$, and then draw his choice $a(t)$ from $p(t)$, where $\mathcal{H}_{t-1} = \{(a(\tau), R_{a(\tau)}(\tau))\}_{\tau \le t-1}$ are the previous arm-reward pairs. 
The goal of the player is to choose the random distribution $p(t)$ properly to maximize the cumulative reward, or minimize the cumulative regret. 
The cumulative regret is defined as the expected reward gap between the real gain and the best one can do, i.e., always selecting the arm with the highest expected reward.
By denoting $a^*=\arg\max_{a\in\mathcal{A}}r_a$ and $r^*=r_{a^*}$, the cumulative regret of policy $\pi$ equals
$Rg(T):=r^*T-\mathbb{E}\left[\sum_{t=0}^{T-1}\sum_{a\in\mathcal A}I_a(t)R_a(t)\right]=\sum_{a:a\neq a^*}(r^*-r_a)\mathbb{E}\left[\sum_{t=0}^{T-1}p_a(t)\right]$,
where $I_a(t)=1$ if and only if arm $a$ is pulled in round $t$. Let $\Delta_a = r^* - r_a$ and assume that $\Delta_a > 0$ for any $a \ne a^*$ (i.e., there is one unique optimal arm), we can write the cumulative regret as $Rg(T)=\sum_{a:a\neq a^*}\Delta_a\mathbb{E}\left[\sum_{t=0}^{T-1}p_a(t)\right]$.

\subsection{Corrupted Bandits}
In a corrupted bandits instance, except for the basic components of the bandit model, there is another adversary who aims to fool the player. 
Specifically, the adversary is aware of the history information as well as the learning policy of the player. However, he cannot obtain the same randomness as used by the user. 
That is, the adversary knows the random distribution $p(t)$ of how the player will choose arm $a(t)$, but does not know the exactly chosen arm $a(t)$.
Based on this knowledge, at each time step $t$, the adversary can change the expected reward of each arm from $\bm{r}$ to $\bm{r}'(t)$, at a cost of $c(t) = \max_{a\in \mathcal{A}} |r_a - r_a'(t)|$. 
In this case, if the player chooses to pull the arm $a(t)$, then his random reward (and feedback) $R_{a(t)}(t)$ is no longer drawn from Bernoulli distribution with mean $r_{a(t)}$, but from Bernoulli distribution with mean $r_{a(t)}'(t)$.

The goal of the adversary is to let the player suffer regret as much as possible, given the constraint that his total cost of corruption could not exceed the corruption level $C$. 
Here, the definition of cumulative regret is the same as classic MAB problems, i.e., we are still comparing the arms under their true expected rewards but not the corrupted expected rewards\footnote{Most of the existing literature uses this definition, e.g.,~\cite{gupta2019better,liu2021cooperative,lykouris2018stochastic,xu2021simple}. As for those who compare the arms under their corrupted expected rewards, e.g.,~\cite{jin2020simultaneously,zimmert2021tsallis}, directly adding $C$ to their regret upper bound leads to a regret bound under our definition.}. 
On the other hand, the goal of the player is to design algorithms such that the regret is still limited even if there is such an adversary. As in many existing works, we assume that the player does not know the corruption level $C$.

Note that our corruption method is slightly different from the existing literature, i.e., the adversary changes the expected reward but not the realized feedback.
In fact, if we use a function to change the realized reward feedback $R$ to $R' = f(R)$ (even for random functions) after seeing the feedback $R \sim \mathcal{D}$, we can get a new reward distribution $\mathcal{D}'$ where $R' \sim \mathcal{D}'$. Hence, our approach (directly changing the reward distribution to $\mathcal{D}'$) is more general than the classic approach. Moreover, the constraint on the adversary in the classic approach is $\sum_t |R_a(t)-R_a'(t)| \le C$, while in our approach it is $\sum_t \left|\mathbb{E}[R_a(t)-R_a'(t)]\right| \le C$, which is looser than the former one. As a result, the adversary in our approach could be more powerful than the classic one with the same $C$. 

\subsection{SAMBA Algorithm}

\begin{algorithm}[ht]%\small
\caption{SAMBA Algorithm}
\label{alg:samba}
\begin{algorithmic}
\Require $\alpha \in (0,1)$
\State \textbf{Init: } $p_a(1):=1/K$ for $a=1,\ldots, K$
\For{$t=1,\ldots, T$}
\State Update the current leading arm $a_{l}\gets \arg\max_a p_a(t)$ 
\State Draw $a(t)$ randomly from probability distribution $p(t)$, and observe the random reward $R_{a(t)}(t)$
\If{$a(t)=a_l$} %for all $a'\neq a_l$
    \State $p_{a'}(t+1)\gets p_{a'}(t)-\alpha p_{a'}^2(t)R_{a(t)}(t)/p_{a_l}(t)$, $\forall a'\neq a_l$
    \Else
    % \State $p_{a(t)}(t+1)\gets p_{a(t)}(t)+\alpha p_{a(t)}^2(t)R_{a(t)}(t)/p_{a(t)}(t)$
    \State $p_{a(t)}(t+1)\gets p_{a(t)}(t)+\alpha p_{a(t)}(t)R_{a(t)}(t)$
    \State $p_a(t+1)\gets p_a(t)$, $\forall a\notin \{a(t), a_l\}$
\EndIf
\State $p_{a_l}(t+1)\gets 1-\sum_{a'\neq a_l} p_{a'}(t+1)$
\EndFor
\end{algorithmic}
\end{algorithm}

The SAMBA algorithm~\cite{denisov2020regret} is described in Algorithm \ref{alg:samba}. 
The policy is a probability distribution vector $p(t)=[p_1(t), \ldots, p_{K}(t)]$ from which an arm is sampled in each round, and is initialized to $p_a(1)=1/K, \forall a\in[K]$. 
In each round $t$, an arm $a(t)$ is sampled from the distribution $p(t)$. The player then pulls arm $a(t)$ and gets a reward $R_{a(t)}(t)$ (a possibly corrupted reward in corrupted bandits).
The probabilities of all the non-leading arms $\forall a\ne a_l$ in $p(t)$ will be updated after the player observes the reward $R_{a(t)}(t)$ according to 
\begin{align}
    p_a(t+1)\gets p_a(t)+\alpha p_a(t)^2\left(\frac{R_a(t)I_a(t)}{p_a(t)}-\frac{R_{a_l}(t)I_{a_l}(t)}{p_{a_l}(t)}\right) \label{eq:update}
\end{align}
where $a_l$ is the current leading arm, i.e., the arm with the highest probability $p_a(t)$. Note that the update scheme applies importance sampling because the player can only observe the reward from the pulled arm. 
After updating the probability of the non-leading arms, the leading arm's probability is given by $p_{a_l}(t+1) = 1 - \sum_{a'\ne a_l} p_{a'}(t+1)$.

\citet{denisov2020regret} prove that SAMBA achieves an $O(\log T)$ regret upper bound in the classic MAB model, which is stated in the following fact. 

\begin{fact}[\citet{denisov2020regret}]\label{fact_1}
    If constant $\alpha<\frac{\Delta}{r^*-\Delta}$, then the SAMBA algorithm for multi-armed bandits problem without corruption ensures a regret 
    $Rg(T)\leq \frac{K}{\alpha\Delta}\log T+Q_0 =  O\left(\frac{K}{\Delta}\log T\right)$,
    % \[Rg(T)\leq \frac{K}{\alpha\Delta}\log T+Q_0 =  O\Big(\frac{K}{\Delta}\log T\Big),\]
    where $Q_0:=\sum_{t=0}^{\infty}\mathbb{P}(p_{a^*}(t)\leq \frac 1 2)<\infty$ is proved in~\citet{denisov2020regret} to be a finite constant. 
\end{fact}

\section{Regret of SAMBA under Corrupted Bandits}\label{sec:mainsec}

Though SAMBA is not specially optimized for the corrupted bandits setting, we surprisingly find out that it works very well even when there is an adversary who tries to fool the algorithm by corruptions.

\begin{theorem}\label{thm:main}
    If constant $\alpha<\frac{\Delta}{r^*-\Delta}$, then the SAMBA algorithm for multi-armed bandits problem with adversarial corruption level $C$ ensures a regret \[Rg(T)=O\left(\frac{K}{\Delta}\log T +\frac C \Delta \right).\]
\end{theorem}

Compared with the existing results, SAMBA achieves a more favorable regret bound by reducing one $\log T$ factor in the existing results (e.g. the $O\big(\frac{K}{\Delta}\log^2 T + C\big)$ bound in~\citet{liu2021cooperative,xu2021simple}), resulting in improved performance as the time horizon $T$ increases. In addition, SAMBA still maintains a linear dependence on the unknown corruption level $C$. %(though with an extra $1/\Delta$ multiplicative term compared with some existing literature). 
This linear term ensures that SAMBA performs well even in the presence of high corruption levels.
Due to the space limit, we only provide some technique highlights here, and defer the whole proof to Appendix~\ref{apdx:proofMainThm}. 

As we have mentioned before, the reason that we are interested in SAMBA is that it is a Markovian policy, in which the influence of one corruption only appears once.
In fact, this is a very important and desired property to deal with corruptions, and most existing anti-corruption algorithms are trying to achieve this property.
For example, BARBAR~\cite{gupta2019better} divide the game into $\log T$ episodes, and only let the corruption in the $i$-th episode influence the arm chosen in the $(i+1)$-th episode.
With this property, we only need to bound the ``sudden'' impact of a corruption, and this substantially reduce the complexity of analysis.

Another good property we found in SAMBA is that the ``sudden'' impact of a corruption scales linearly with the corruption cost. 
Roughly speaking, if there is a corruption with cost $c(t)$ at time step $t$ and no corruptions after $t$, then it only requires about $\Theta(c(t))$ steps to counteract its influence, i.e., the probability distribution $p(t+d\cdot c(t))$ for some constant $d$ becomes close to $p(t)$ (the probability distribution before corruption) as the corruption effect is mostly counteracted in $d\cdot c(t)$ steps. 
Then, by the Markovian property of SAMBA, we could imagine that the regret incurred by corruption is approximately the regret in the next $d\cdot c(t)$ steps, and hence also scales linearly with $c(t)$.
In this way, we can finally show that the corruption dependent term of SAMBA is linear with $C$.

To better understand the above ideas, we first briefly recall how the analysis (without corruption) in~\citet{denisov2020regret} works. They divide the learning procedure into two cases: i) the case that $p_{a^*}\geq 1/2$; and ii) the case that $p_{a^*}<1/2$.
Our analysis on SAMBA for the corrupted bandits problem also follows these two cases. 

\subsection{The case when $p_{a^*}<1/2$ }
\label{sec:psmall}

When there is no corruption, \citet{denisov2020regret} consider the case where the optimal arm $a^*$ is not the leading arm $a_l$, and use $\E[p_{a^*}^{-1}(t)]$ to capture the trajectory of how $p_{a^*}(t)$ changes during the learning procedure. 
Specifically, when $a^*$ is not $a_l$, from some calculations according to SAMBA's policy update rule, one can show that

\begin{eqnarray*}
p_{a^*}^{-1}(t+1)=\left\{
\begin{aligned}%\footnotesize
&p_{a^*}^{-1}(t)-\frac{\alpha}{1+\alpha}p_{a^*}^{-1}(t) && \ \ \text{ w.p. }\  r^*p_{a^*}(t)\\
&p_{a^*}^{-1}(t)+\alpha\frac{p_{a^*}^{-1}(t)}{p_{a_l}(t)p_{a^*}^{-1}(t)-\alpha} && \ \ \text{ w.p. }\ r_{a_l}p_{a_l}(t)\\
&p_{a^*}^{-1}(t) && \ \ \text{ otherwise}
\end{aligned}
\right.
\end{eqnarray*}
Thus, when there is no corruption,
\begin{align}
    \mathbb{E}[p_{a^*}^{-1}(t+1)|H(t)]-p_{a^*}^{-1}(t)=&\alpha r_{a_l}\frac{p_{a_l}(t)p_{a^*}^{-1}(t)}{p_{a_l}(t)p_{a^*}^{-1}(t)-\alpha}-\frac{\alpha r^*}{1+\alpha}   \notag  \\
    \leq& \alpha(r^*-\Delta)(1+\epsilon)-\frac{\alpha r^*}{1+\alpha}
\end{align}
where the last inequality holds because for the leading arm, $p_{a_l}(t)>1/K$ and $r_{a_l}\leq r^*-\Delta$, and the constant $\epsilon>0$ is chosen to satisfy $(1+\epsilon)(1+\alpha)<\frac{r^*}{r^*-\Delta}$. Such $\epsilon>0$ must exist because $\alpha<\frac{\Delta}{r^*-\Delta}$. 
Let constant $\xi:=\alpha \frac{r^*}{1+\alpha}-\alpha (r^*-\Delta)(1+\epsilon)>0$, we can get 
\begin{equation}\label{eq:smallpInduction}
    \mathbb{E}\Big[p_{a^*}^{-1}(t+1)\Big|H(t)\Big]-p_{a^*}^{-1}(t)\leq -\xi. 
\end{equation}

Note that at the beginning of the algorithm, $\E[p_{a^*}^{-1}(0)] = {1/ K^{-1}} = K$. When $p_{a^*}^{-1}(t) \leq 2$, the arm $a^*$ must be the leading arm. 
Hence, after at most $\lceil\frac{K-2}{\xi}\rceil$ steps, $\E[p_{a^*}^{-1}(t)]$ can become small enough, which results in $p_{a^*}(t)$ being large enough and $a^*$ becoming the leading arm.
Furthermore, if $a^*$ again becomes a non-leading arm, say at time $t'$, then $\E[p_{a^*}^{-1}(t')]$ is likely to be smaller than $\E[p_{a^*}^{-1}(0)]$ because in expectation the probability of sampling non-optimal arms $p_a (a\neq a^*)$ would be updated to a smaller value then. Thus, from the Markov property, the expected number of steps needed for $a^*$ to become the leading arm again is smaller than the number of steps needed in the first time. The same reasoning applies to the future ``non-leading to leading'' transitions. 
From such intuition, they prove the regret that occurs when $p_{a^*}<\frac 1 2$ (not only when $a^*$ is a non-leading arm) can be upper bounded by some constant $Q_0$, referring to~\cite{denisov2020regret} for details.

Now let's consider what happens if there is an adversary to deploy corruptions.
We also consider the case where $a^*$ is not the leading arm first. In such case, $p_{a^*}^{-1}(t+1)$ is updated as
\begin{eqnarray*}
\ \ \ \ \ \left\{
\begin{aligned}
&p_{a^*}^{-1}(t)-\frac{\alpha}{1+\alpha}p_{a^*}^{-1}(t) && \ \ \text{ w.p. }\  r_{a^*}'(t) p_{a^*}(t)\\
&p_{a^*}^{-1}(t)+\alpha\frac{p_{a^*}^{-1}(t)}{p_{a_l}(t)p_{a^*}^{-1}(t)-\alpha} && \ \ \text{ w.p. }\ r_{a_l}'(t)p_{a_l}(t)\\
&p_{a^*}^{-1}(t) && \ \ \text{ otherwise}
\end{aligned}
\right.
\end{eqnarray*}
Then, we can derive
\begin{align*}
&\mathbb{E}[p_{a^*}^{-1}(t+1)|H(t)]-p_{a^*}^{-1}(t)\\
\leq& \alpha (r_{a_l}+c(t))\frac{p_{a_l}(t)p_{a^*}^{-1}(t)}{p_{a_l}(t)p_{a^*}^{-1}(t)-\alpha}-\frac{\alpha (r^*-c(t))}{1+\alpha}\notag\\
\nonumber \leq& \alpha(r^*-\Delta+c(t))(1+\epsilon)-\frac{\alpha (r^*-c(t))}{1+\alpha}\\
=& -\xi+\alpha c(t)\big(1+\epsilon+\frac{1}{1+\alpha}\big)\label{eqn:corr1}
\end{align*}

where the first inequality is because $r'_{a_l}(t) \le r_{a_l} + c(t)$ and $r_{a^*}(t) \ge r_{a^*} - c(t)$.
We can see that, except for the regular bias $-\xi$, the corruption $c(t)$ can increase $\E[p_{a^*}^{-1}(t)]$ by at most $\alpha c(t)\big(1+\epsilon+\frac{1}{1+\alpha}\big)$.
Hence, one can imagine that if the total corruption level is upper bounded by $C$, then we need to run the algorithm for an extra number of $O(\frac{C}{\xi})$ time steps to counteract the influence of corruption.
Here, we omit the technical details, which are relegated to Appendix~\ref{apdx:proofMainThm}.

\begin{figure}[htbp]
\centering
\begin{minipage}[t]{0.5\textwidth}
\centering
\includegraphics[width=0.9\linewidth]{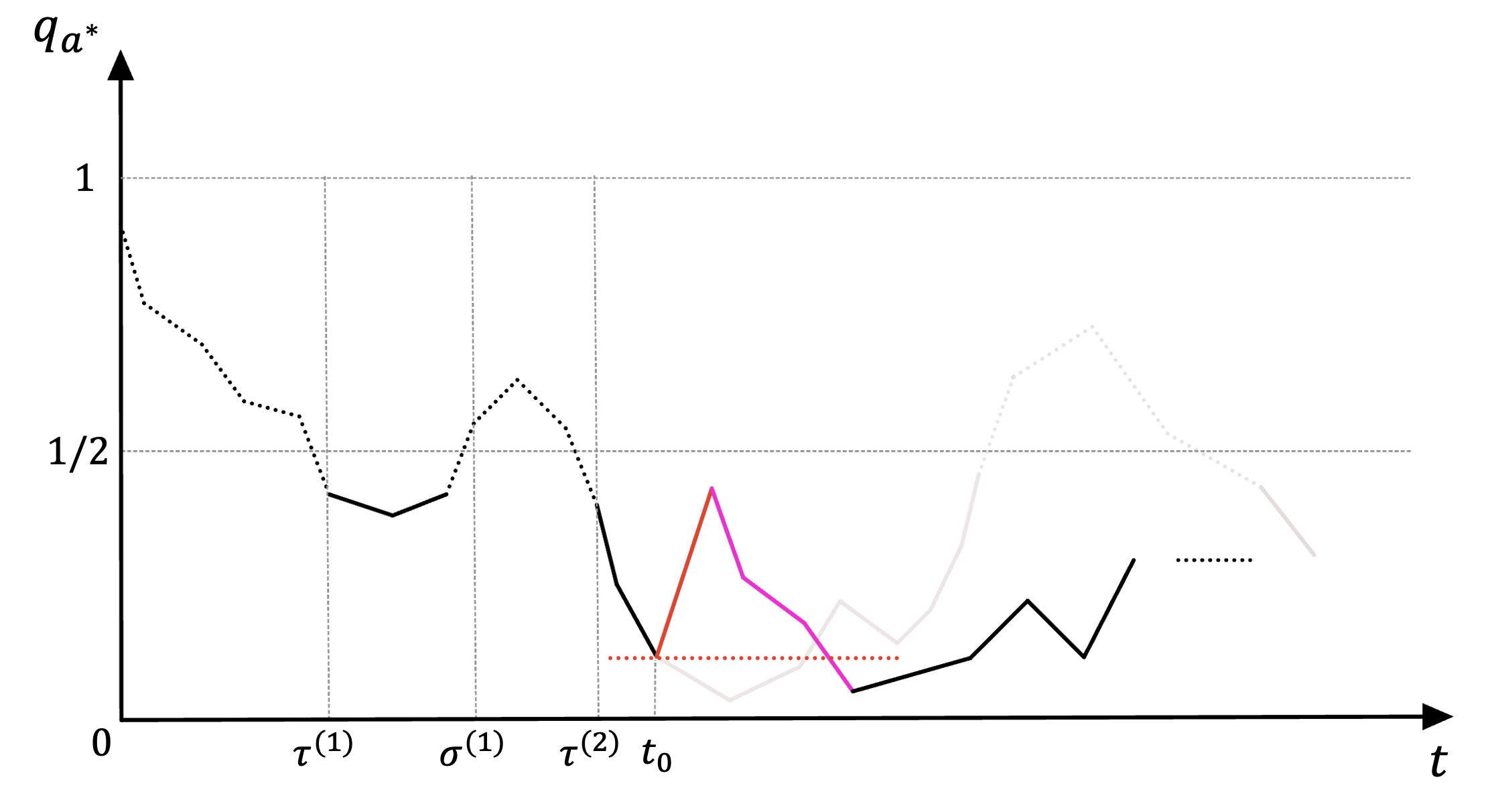}
\caption{Recovery process. }
\label{fig:embedded3-1}
\end{minipage}
\begin{minipage}[t]{0.4\textwidth}
\centering
\includegraphics[width=0.8\linewidth]{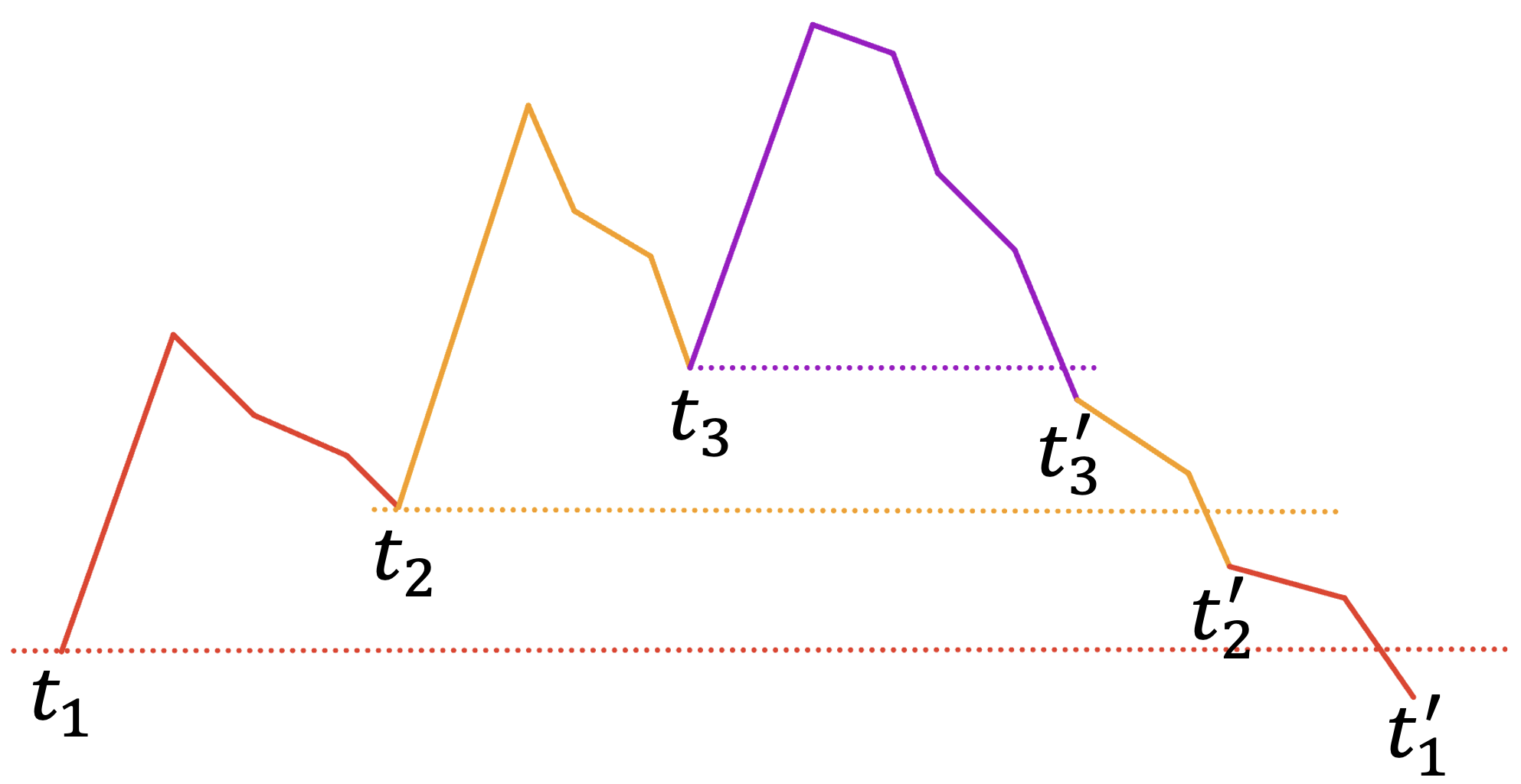}
\caption{Consecutive corruptions.}
\label{fig:consecutive1-1}
\end{minipage}
\end{figure}

\subsection{The case when $p_{a^*}\geq 1/2$}\label{sec:plarge}

On the other hand, if $p_{a^*}\geq 1/2$, then $a^*$ must be the leading arm. In this case,~\cite{denisov2020regret} uses $\E[p_a(t)]$ to capture the trajectory of how $p_{a}(t)$ changes during the learning procedure for each arm $a \ne a^*$.
Specifically, one can show that for $a\neq a^*$, 
\begin{equation}\label{eq:update1}
    \mathbb{E}[p_a(t+1)-p_a(t)|H(t)]\leq \alpha p_a(t)^2(r_a-r_{a^*})\leq -\alpha\Delta p_a(t)^2.
\end{equation}
Let $q_{a^*}:=1-p_{a^*}=\sum_{a:a\neq a^*}p_a$. From (\ref{eq:update1}) and Jensen's inequality, 
\begin{equation} 
    \mathbb{E}[q_{a^*}(t+1)|H(t)]-q_{a^*}(t)\leq \sum_{a:a\neq a^*}-\alpha\Delta p_a(t)^2\leq -\frac{\alpha\Delta}{K}q_{a^*}(t)^2.
\end{equation}
Thus, $\E[q_{a^*}(t)]$ is in the same order as $\frac{K}{2K+\alpha\Delta t}$ (whose trajectory can be easily verified to satisfy the above equation).
Hence, by taking the sum, the regret that occurs when $p_{a^*}\geq 1/2$ is upper bounded by $O(\frac{K}{\alpha\Delta}\log T)$. Details can be found in Appendix~\ref{apdx:proofMainThm}.

Now, we consider the case where there are adversarial corruptions. 
When $p_{a^*}\geq 1/2$, then $a^*$ is the leading arm and for any $a\ne a^*$, we have 
\begin{align*}
    \mathbb{E}[p_a(t+1)-p_a(t)|H(t)]&\leq \alpha p_a(t)^2\Big((r_a+c(t))-(r_{a^*}-c(t))\Big)\\
    &\leq \alpha(2c(t)-\Delta) p_a(t)^2.
\end{align*}
That is, except for regular bias $-\Delta_ap_a^2(t)$, the corruption $c(t)$ can increase $\E[p_a(t)]$ for at most $2\alpha c(t)p_a^2(t)$.
However, this increase is not a constant, and one cannot directly obtain how many time steps are needed to counteract the influence of corruption.
The trick here is to notice that after corruption, $p_a$ becomes larger, and hence its decreasing rate $\alpha p_a^2$ becomes larger than $\alpha p_a^2(t_c)$ before it fully recovers from the corruption, where $t_c$ is the time step that the corruption occurs. 
Formally, we first define the recovery process as follows. 
\begin{definition}[Recovery process]
    The recovery process of a corruption at time $t_c$ is a time interval $[t_c+1, t_c']$ on process $\{q_{a^*}(t)\}$ such that $t_c'$ is the first time step satisfying $t_c'\geq t_c+1$ and $q_{a^*}(t_c')\leq q_{a^*}(t_c)$. 
\end{definition}
Roughly speaking, the recovery process of corruption at time $t_c$ is the time steps required to let $q_{a^*}(t)$ fall below $q_{a^*}(t_c)$.

\begin{figure*}[htb]
\centering
\begin{subfigure}{0.49\linewidth}
    \centering
    \includegraphics[width=0.9\textwidth]{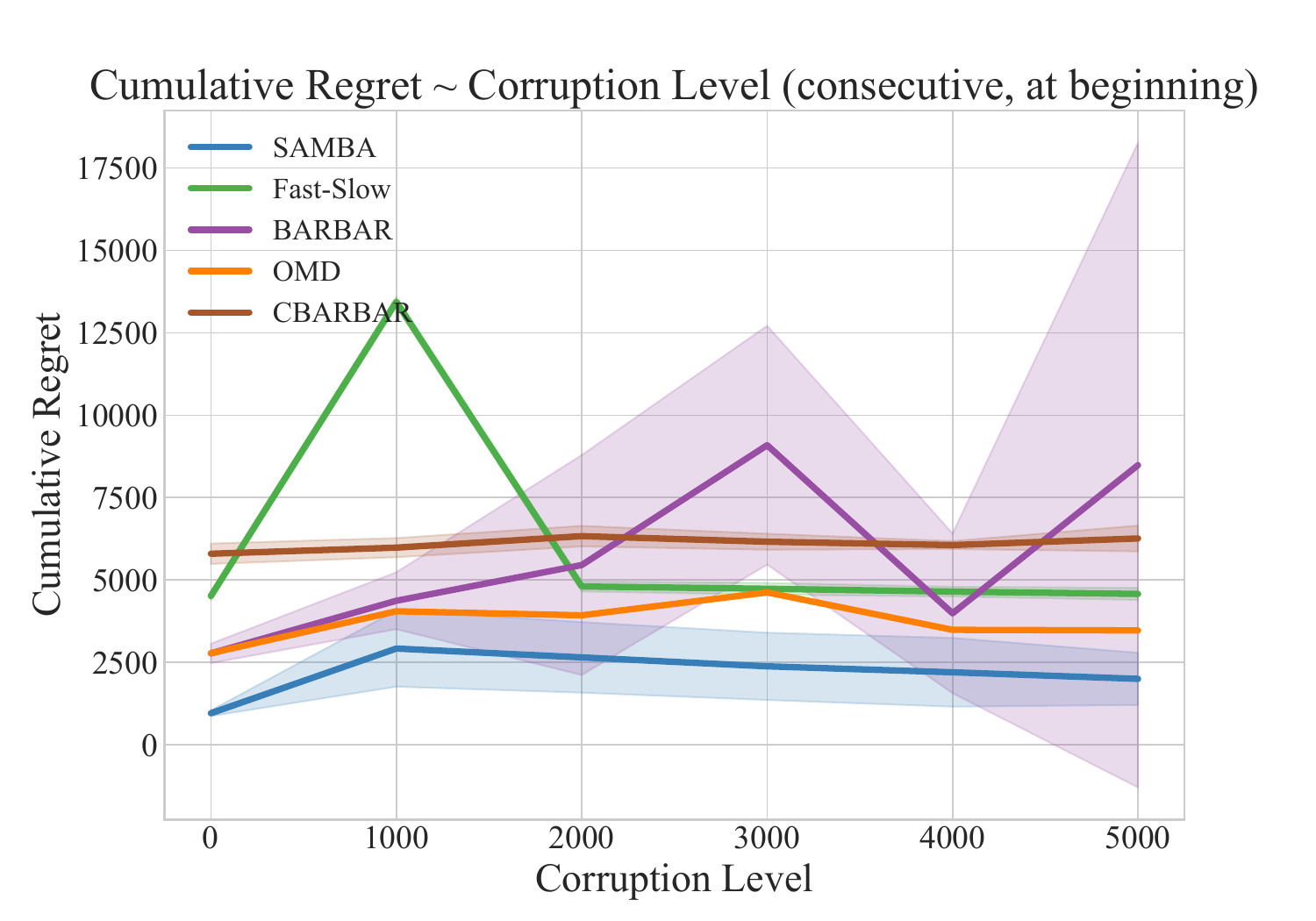}
    \caption{Corruption scheme 1: consecutive, at beginning.}
    \label{fig:first}
\end{subfigure}
% \hfill
\begin{subfigure}{0.49\linewidth}
    \centering
    \includegraphics[width=0.9\textwidth]{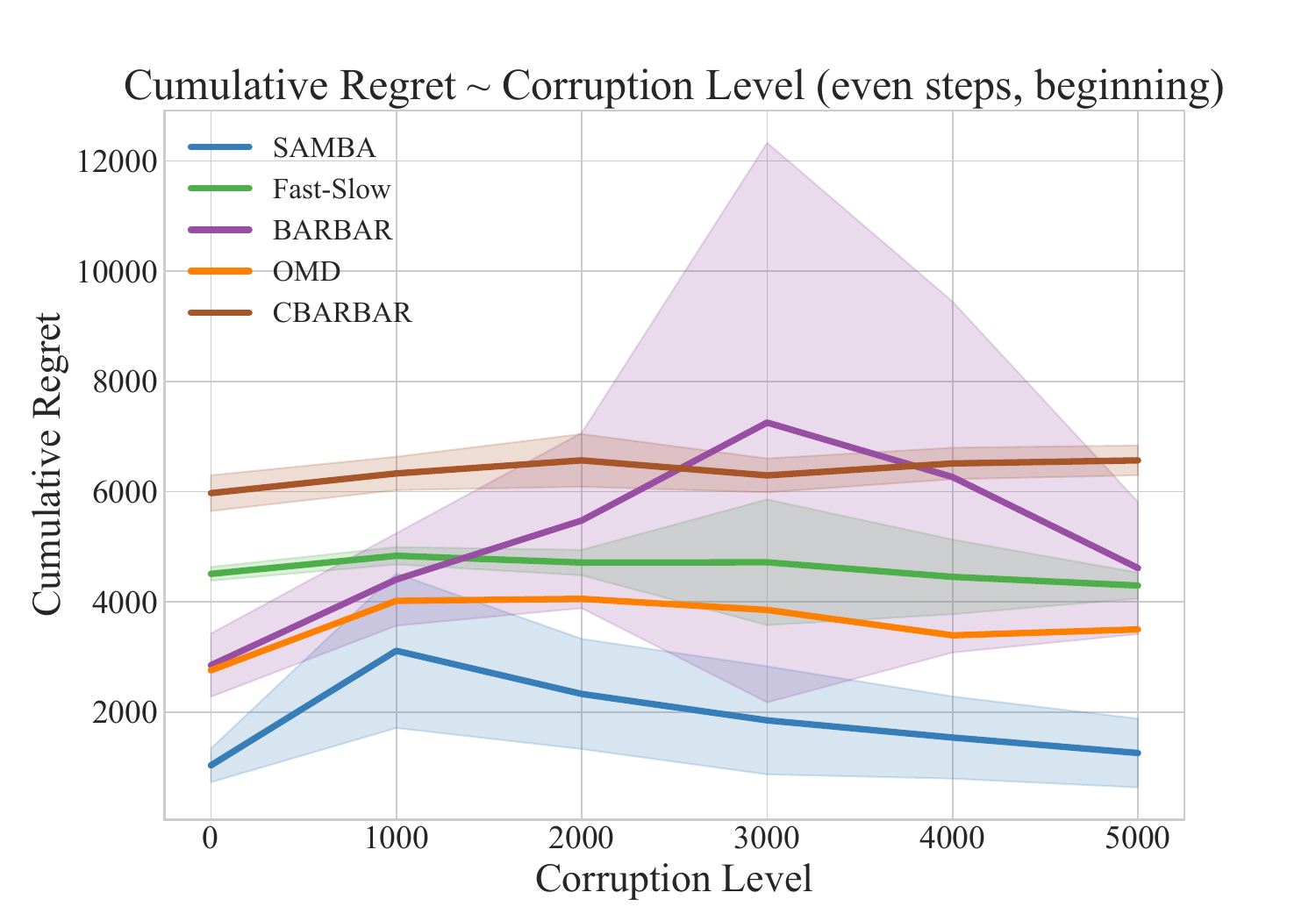}
    \caption{Corruption scheme 2: even steps, at beginning.}
    \label{fig:second}
\end{subfigure}
% \hfill
\begin{subfigure}{0.49\linewidth}
    \centering
    \includegraphics[width=0.9\textwidth]{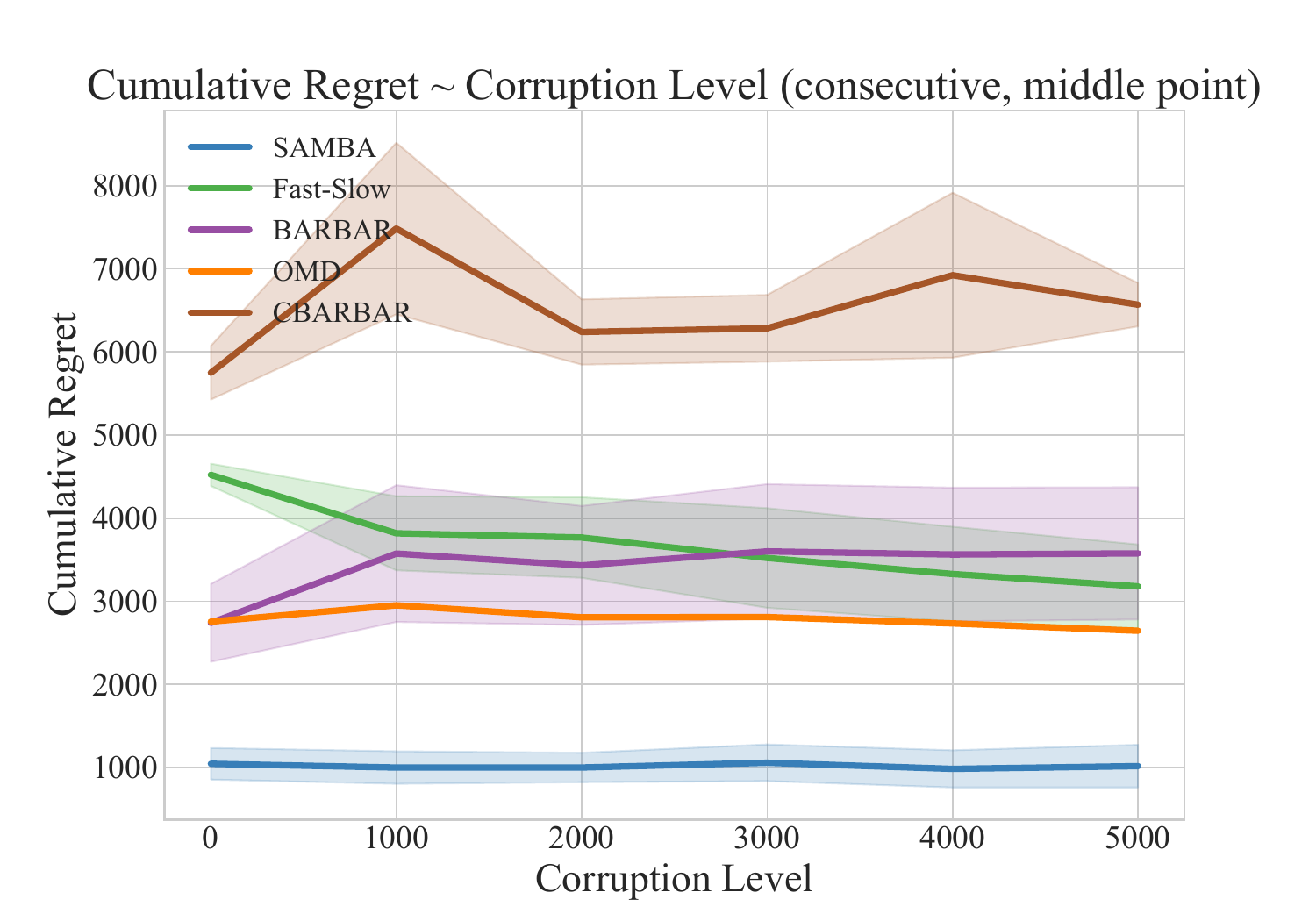}
    \caption{Corruption scheme 3: consecutive, in the middle.}
    \label{fig:third}
\end{subfigure}
% \hfill
\begin{subfigure}{0.49\linewidth}
    \centering
    \includegraphics[width=0.9\textwidth]{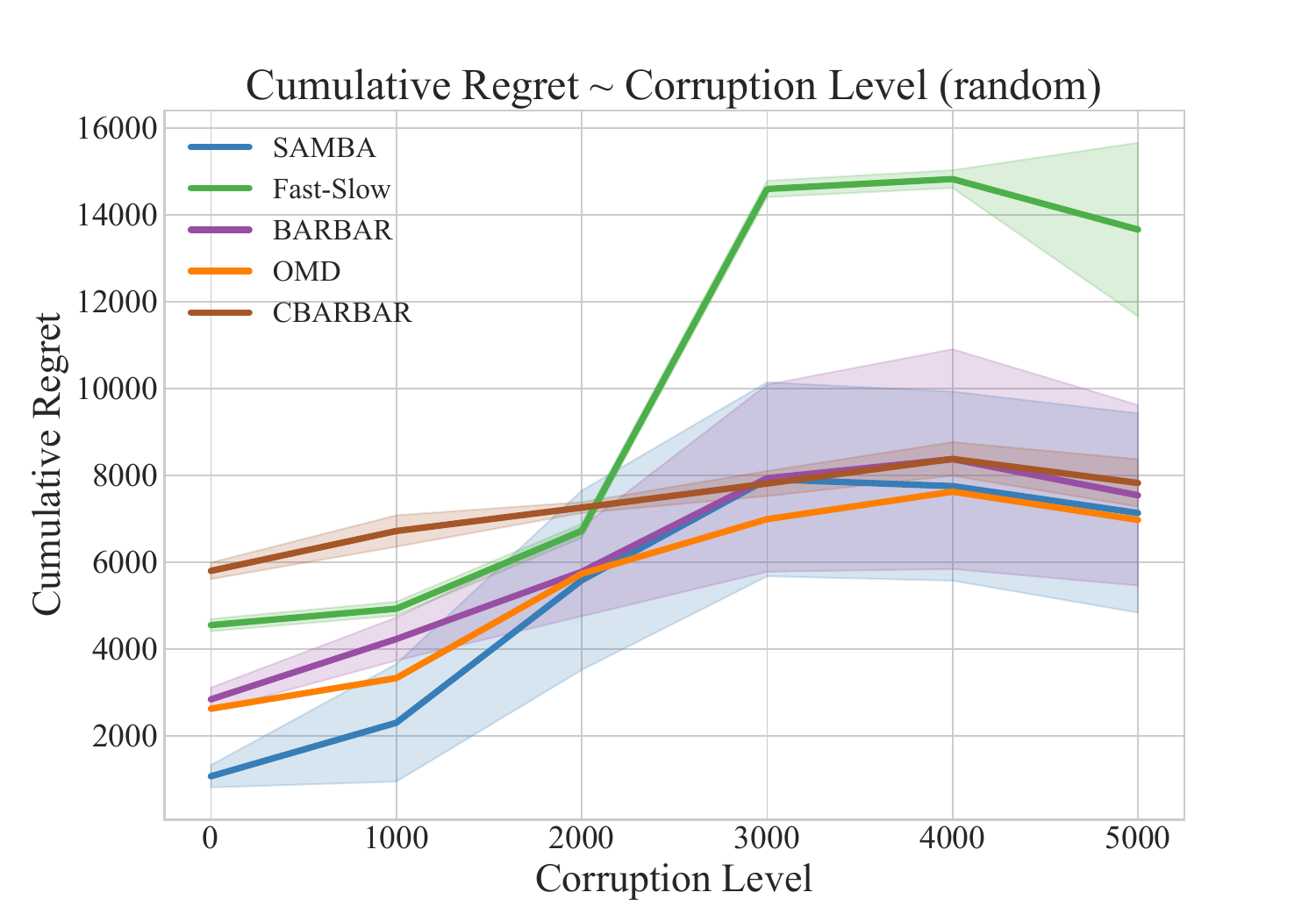}
    \caption{Corruption scheme 4: at random steps. }
    \label{fig:fourth}
\end{subfigure}
\caption{Comparison of different algorithms: the cumulative regrets under different corruption levels and different corruption schemes. 
SAMBA achieves the lowest cumulative regret in most settings, particularly outperforming baselines when \( C=0 \), demonstrating its \( O(\log T) \) regret versus \( O(\log^2 T) \) for others. However, as corruption \( C \) increases, SAMBA's advantage diminishes, consistent with its regret bound of \( O(C + \log T) \), while OMD shows worse performance due to its high complexity and large constant factors.}
\label{fig:regretCorr}
\end{figure*}

If there is a large corruption in only one step, say step $t_0$ with corruption level $c(t_0)>\Delta/4$, then $\mathbb{E}[q_{a^*}]$ may increase after $t_0$ and subsequently gradually decrease, as shown in Figure~\ref{fig:embedded3-1}. 
What we want to do is to upper bound the expected number of steps during the recovery process after corruption $c(t_0)$ (colored magenta in Figure~\ref{fig:embedded3-1}). 
Here we use the optional stopping theorem to give such a bound. 
Let $\phi=\min\{t> t_0:q_{a^*}(t)\leq q_{a^*}(t_0)\}$. When $t_0<t\leq \phi$, it holds that $q_{a^*}(t)\geq q_{a^*}(t_0)$. Then, 
\[\mathbb{E}[q_{a^*}(t+1)|H(t)]-q_{a^*}(t)\leq-\frac{\alpha \Delta}{2K}q_{a^*}(t)^2\leq -\frac{\alpha \Delta}{2K}q_{a^*}(t_0)^2.\] 
Thus, $\{q_{a^*}(t)|t> t_0\}$ is a supermartingale. From the optional stopping theorem, 
\begin{align*}
    &\ \mathbb{E}[q_{a^*}(\phi\wedge t)]+\frac{\alpha \Delta}{2K}q_{a^*}(t_0)^2\mathbb{E}[\phi\wedge t]\\
\leq&\ \mathbb{E}[q_{a^*}(\phi\wedge (t_0+1))]+\frac{\alpha \Delta}{2K}q_{a^*}(t_0)^2\mathbb{E}[\phi\wedge (t_0+1)].
\end{align*}
Here, $\wedge$ denotes the pairwise minimum. 
Then, applying the monotone converge theorem, we get 
\begin{align}
&\ \mathbb{E}[\phi-t_0-1]\leq \lim_{t\rightarrow \infty}\mathbb{E}[\phi\wedge t]-\mathbb{E}[\phi\wedge (t_0+1)]\notag\\
\leq&\ \frac{2K}{\alpha\Delta q_{a^*}(t_0)^2}(\mathbb{E}[q_{a^*}(t_0+1)]-\mathbb{E}[q_{a^*}(\phi)]).\label{eq:largepInduction2}
\end{align}
Therefore,
\begin{align}
\mathbb{E}[\phi-t_0]\leq& \frac{2K}{\alpha\Delta q_{a^*}(t_0)^2}\Big((2c(t)-\Delta)\frac{\alpha}{K}q_{a^*}(t_0)^2+\frac{\alpha \Delta}{2K}q_{a^*}(t_0)^2\Big)+1 \notag \\
=&\frac{4c(t)}{\Delta}. \label{eq:largepInduction}
\end{align}

If there are consecutive corruptions (other corruptions come before recovery from the previous corruption), then the total extra regret incurred by these corruptions is upper bounded by the regret calculated by considering these corruption steps separately from ``inner'' corruptions to ``outer'' corruptions. 
Here we use Figure~\ref{fig:consecutive1-1} as an example. Corruptions are made at the time steps $t_1, t_2, t_3$. We first deal with $c(t_3)$, then $c(t_2)$, and finally $c(t_1)$. The length of the recovery process for $c(t_3)$ (colored purple) can be bounded directly by the previous derivation. After dealing with $(t_3,t_3']$, we can remove this interval and combine the rest together. Then, we consider the interval $(t_2,t_3]\cup (t_3', t_2']$ as a whole and apply the optional stopping theorem to it, which holds because $q_{a^*}(t_3')\leq q_{a^*}(t_3)$. The same analysis holds for $c(t_1)$ (details can be found in Appendix~\ref{apdx:proofMainThm}). 
In this way, we can upper bound the expected number of total recovery steps needed by $\sum_{t=0}^{T-1}\frac{4c(t)}{\Delta}=\frac{4C}{\Delta}$.

From the above analysis, we know that in both cases ($p_{a^*} < 1/2$ or $p_{a^*}\geq 1/2$), the influence of corruption $C$ would be counteracted by $O(C)$ time steps, leading to an additional regret of $O(C)$. Therefore, along with Fact \ref{fact_1}, we can get the final regret upper bound as $O(K\log T + C)$. The formal proofs can be found in Appendix~\ref{apdx:proofMainThm}. 

\begin{remark}
    In BARBAR (and CBARBAR), the algorithm is divided into $\log T$ phases (the length of each phase keeps doubling). To ensure that the algorithm is robust against corruptions, any arm should be pulled $O(\log T)$ times in each phase (so that the empirical mean is accurate enough) to detect the corruptions. This leads to a $O(\log^2 T)$ regret even when there is no corruption. 
    In SAMBA, the algorithm and analysis are based on expectations but not accurate empirical means. % ; hence, they are more sensitive. 
    Thus, we do not require pulling each arm $O(\log T)$ times in each phase to detect the influence of corruption, and instead, only a constant number of pulls in each phase is enough. In this way, we reduce one $\log T$ factor in the regret upper bound (note that when there is no corruption, every arm is pulled $\Theta(\log T)$ times, which is enough to guarantee good performance). 
\end{remark}

\section{Simulation}\label{sec:simulation}

We then conduct experiments to compare the empirical performance of SAMBA with four baseline algorithms. 
We set the parameters to $T=100,000$, $K=9$ and the 9 arms are of mean rewards $0.1, 0.2, \ldots, 0.9$ respectively, $\alpha=0.05$ in SAMBA, and $\delta=1/T$ in Fast-Slow AAE. 
We test with five different corruption levels $C = 1000, 2000, \ldots, 5000$ and on four different corruption schemes:

\begin{enumerate}
\setlength{\itemsep}{1pt}
\setlength{\parsep}{0pt}
\setlength{\parskip}{0pt}
    \item All corruption added at the beginning consecutively, i.e., at steps $0, 1, 2, 3, \ldots$;
    \item All corruption added at the the even steps at the beginning, i.e., at steps $0, 2, 4, 6, \ldots$; 
    \item All corruption added concentratedly in the middle, i.e., at steps $T/4, T/4 + 1, T/4+2, \ldots$; 
    \item All corruptions added at random steps among the first $T/10=10,000$ steps. 
\end{enumerate}

\begin{figure}[tbh]
\begin{subfigure}{0.9\linewidth}\center
    \includegraphics[width=\textwidth]{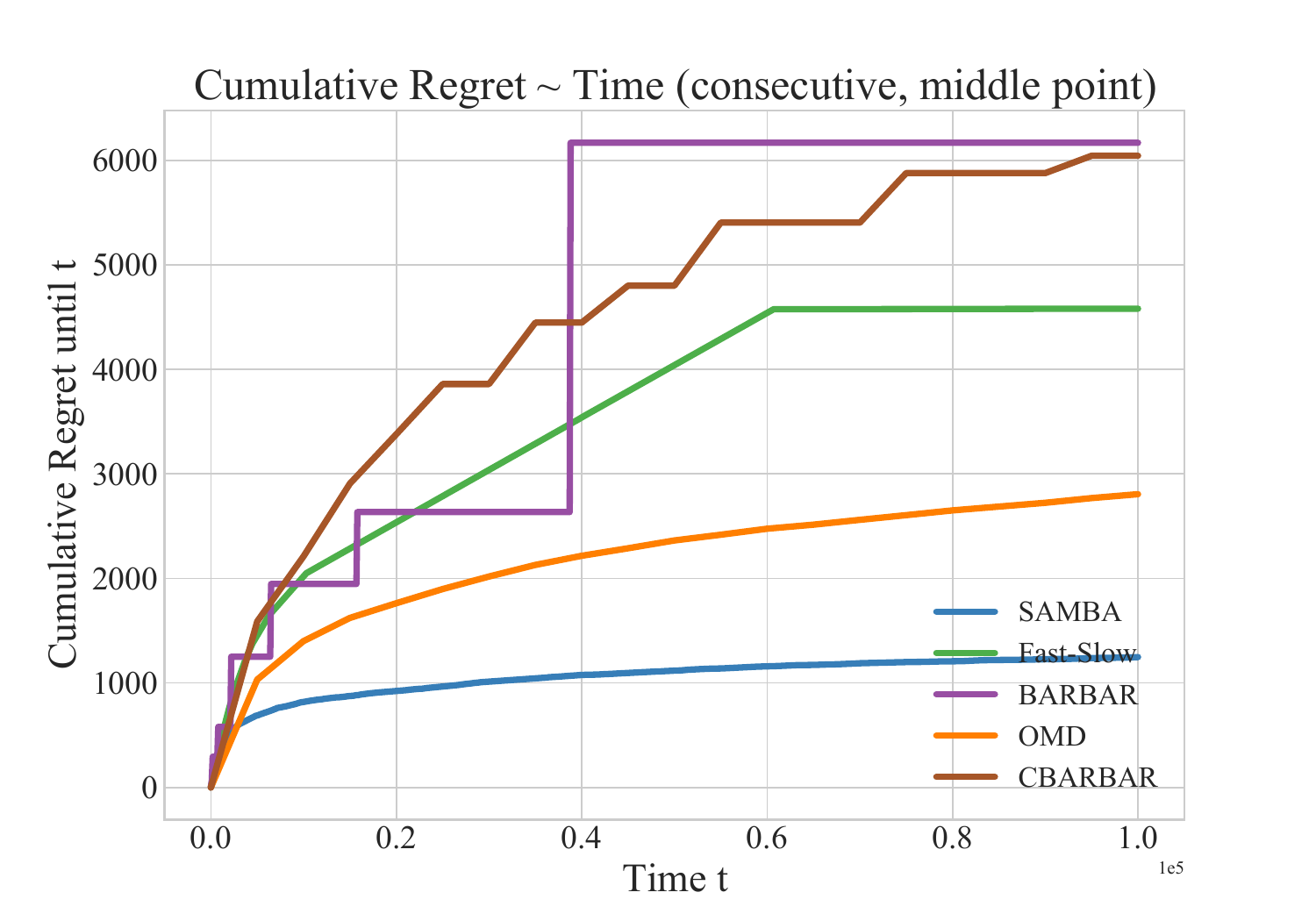}
    \caption{Corruption scheme 3: consecutive, in the middle.}
    \label{fig:first2}
\end{subfigure}
% \hfill
\begin{subfigure}{0.9\linewidth}\center
    \includegraphics[width=\textwidth]{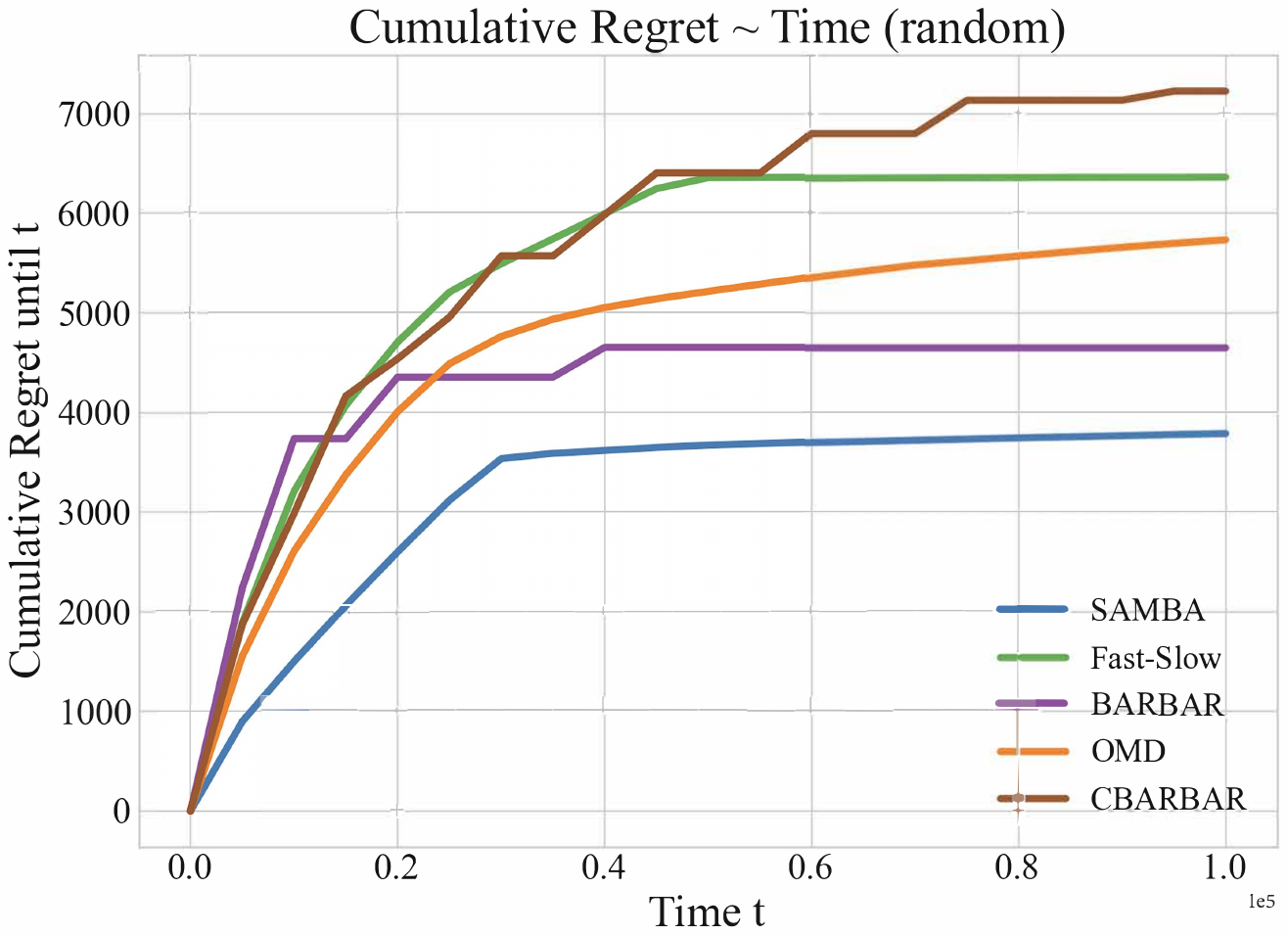}
    \caption{Corruption scheme 4: at random steps.}
    \label{fig:second2}
\end{subfigure}
\caption{Comparison of different algorithms: the trend of their cumulative regret with the time when $C=2000$ under corruption schemes 3 and 4.}
\label{fig:regretTime}
\end{figure}

\begin{table}[tbh] 
    \centering
    \caption{The average single-run time (in second) and standard deviation (SD) of different algorithms. }\label{table:time}
    \begin{footnotesize}
    \begin{tabular}{|c|c|c|c|c|c|}
        \hline
        & SAMBA & Fast-Slow & BARBAR & CBARBAR & OMD \\
        \hline
        Time&2.2594  &  1.2942 & 0.7401  & 0.7823 & 1733.3\\
        \hline
        SD &  0.01422  &  0.01057 &0.00701& 0.00684 & 10.903 \\
        %\hline
        %Time&2.26  &  1.29 & 0.74  & 0.78 & 1733.3\\
       % \hline
        %SD &  0.014  &  0.011 &0.007& 0.007 & 10.9 \\
        \hline
    \end{tabular}
    \end{footnotesize}
\end{table} 
First, we compare the time costs of these corrupted bandits algorithms, and the results are shown in Table~\ref{table:time}. %(running on an Apple M2 Pro machine).
We can see that the combinatorial algorithms have a much lower time cost than the OMD methods, e.g., SAMBA runs more than 500x faster than OMD. This indicates the efficiency of our algorithm, i.e., it is a \emph{combinatorial} algorithm with asymptotically optimal regret upper bound.

Then, we consider the cumulative regret under different corruption levels. 
The experiment result is shown in Figure~\ref{fig:regretCorr}. 
It shows the mean and standard deviation of the cumulative regret for the four algorithms under different settings. Each experiment runs for 100 times, except the one on the OMD algorithm which runs very slow due to its requirement of solving an optimization problem in each step. 
We can see that SAMBA performs the best in terms of cumulative regret in most settings.
Specifically, when $C=0$, SAMBA outperforms the baselines, which demonstrates SAMBA's $O(\log T)$ regret advantage over other algorithms' $O(\log^2 T)$ regret. 
However, it seems that SAMBA's performance advantage over baseline algorithms decreases as the corruption level $C$ increases. This actually matches SAMBA's regret bound of $O(C+\log T)$. When $C$ is large, the regret is determined primarily by $C$ rather than the $\log T$ term. 
As for OMD, it has a much higher time complexity, and performs worse than SAMBA when the corruption level is small, because some large constant factors appear in the regret upper bound. 

In addition, we compare the cumulative regret of different algorithms over time. The curves for two corruption schemes and corruption level $C=2,000$ are shown in Figure~\ref{fig:regretTime}. Here, BARBAR and CBARBAR are implemented by selecting $n_a(m)$ times of arm $a$ in phase $m$, where $n_a(m)$ is predetermined before phase $m$. Thus, the non-optimal arms are sampled together, leading to a step-like curve. 
In Figure~\ref{fig:first2}, the consecutive corruptions in the middle incur a regret surge (a large number of non-optimal arms selected after the concentrated corruptions) for BARBAR, while SAMBA actually converges quickly and tolerates the abrupt corruptions in the middle well. 

\balance

We also conduct experiments with varying numbers of arms to compare the performance of different algorithms. The tested values of $K$ (number of arms) are 6, 8, 10, 15, 20, and 30. The mean reward for each arm is uniformly distributed in the range $[0,1]$. The mean cumulative regrets (with $T=100,000$) are summarized in Table~\ref{table:diffNumberArms}. The corruption level is set to $C=3,000$, with corruptions concentrated in the middle of the time horizon (corruption scheme 3).

\begin{table}[tbh]
    \centering
    \caption{Comparison of the mean cumulative regret under different algorithms and different number of arms (with corruption level 3000 and corruption scheme 3)}
    \label{table:diffNumberArms}
    \begin{footnotesize}
    \begin{tabular}{|c|c|c|c|c|c|c|}
        \hline
        \diagbox[width=7.5em]{Algorithm}{\\ $K$} & 6 & 8 & 10 & 15 & 20 & 30 \\ \hline  % ,font=\scriptsize
        % No. of arms & 6 & 8 & 10 & 15 & 20 & 30 \\ \hline
        SAMBA & 629.9 & 884.2 & 1054.8 & 1722.7 & 2947.4 & 3534.4\\ \hline
        Fast-Slow & 1473.7 & 2265.3 & 3519.8 & 7955.8 & 10269.1 & 12661.2\\ \hline
        BARBAR & 2010.5 & 8813.9 & 3599.2 & 4800.7 & 8675.4 & 10695.7\\ \hline
        CBARBAR & 4189.2 & 6901.2 & 6285.2 & 11874.3 & 12668.2 & 14644.0\\ \hline
        OMD & 2038.3 & 2839.5 & 3322.7 & 4575.6 & 8460.1 & 10189.2\\ \hline
    \end{tabular}
    \end{footnotesize}
\end{table}

The experimental results indicate that as the number of arms $K$ increases, the cumulative regrets scale approximately linearly with $K$, aligning well with the theoretical bounds. In all cases, SAMBA achieves the lowest mean cumulative regret, consistently outperforming the other algorithms.

\section{Conclusion and Future Directions}

In this paper, we apply a policy gradient algorithm SAMBA to the stochastic multi-armed bandits problem with adversarial corruptions.
Our analysis is the first result of a combinatorial algorithm that achieves an asymptotically optimal regret upper bound of $O(C+\log T)$, establishing our method as the state-of-the-art in the corrupted bandits setting. 
We have also conducted simulations, demonstrating that SAMBA outperforms existing baselines. 

There are several directions for future work. 
For example, it would be interesting to generalize SAMBA as well as our analysis to the combinatorial bandit setting or linear bandit setting, and it would also be valuable to validate the algorithm's performance in real-world applications, e.g., to conduct experiments on actual systems or design large-scale simulations that capture realistic complexities.

\begin{acks}
The work of Siwei Wang is supported in part by the National Natural Science Foundation of China Grant 62106122. The work of Zhixuan Fang is supported by Tsinghua University Dushi Program and Shanghai Qi Zhi Institute Innovation Program SQZ202312. 
\end{acks}

% \clearpage
% \normalem
    
%%%%%%%%%%%%%%%%%%%%%%%%%%%%%%%%%%%%%%%%%%%%%%%%%%%%%%%%%%%%%%%%%%%%%%%%

%%% The next two lines define, first, the bibliography style to be 
%%% applied, and, second, the bibliography file to be used.

\bibliographystyle{ACM-Reference-Format} 
\bibliography{main}

\newpage
%%%%%%%%%%%%%%%%%%%%%%%%%%%%%%%%%%%%%%%%%%%%%%%%%%%%%%%%%%%%%%%%%%%%%%%%

\clearpage

\appendix

\section{Proof of Theorem~\ref{thm:main}}\label{apdx:proofMainThm}

In the following proofs, we use notations $\vee, \wedge$ to denote the pairwise maximum and minimum, respectively. 
For $c(t)> \Delta/4$, we call it a large corruption; otherwise ($c(t)\leq \Delta/4$), it is a small corruption. 

\begin{definition}[Recovery process]
    The recovery process of a corruption at time $t_c$ is an embedded chain of $\{q_{a^*}(t)\}$ with time interval $t\in[t_c+1, t_c']$ where $t_c'$ is the first time step satisfying $t\geq t_c+1$ and $q_{a^*}(t)\leq q_{a^*}(t_c)$. 
\end{definition}
After the recovery process, the negative influence of corruption $c(t_0)$ on $q_{a^*}$ will be eliminated. 
See Figure~\ref{fig:embedded3} as an example, the red line segment represents the influence after the corruption step $t_0$, the magenta line segments represent the gradual recovery from corruption. The red dashed line represents a horizontal level; the recovery process completes when a magenta line segment intersects the red dashed line, or equivalently, after such an intersection, $q_{a^*}(t_0')$ will be no larger than $q_{a^*}(t_0)$. 

Denote the steps with large corruption as $t^{(LC)}_1, t^{(LC)}_2, \ldots, t^{(LC)}_w$. 
Define $S_{LC}=\cup_{i}(t^{(LC)}_i,t'^{(LC)}_i]$, where $t^{(LC)}_i$'s are the steps with large corruptions and $(t^{(LC)}_i,t'^{(LC)}_i]$ is the recovery process after the corruption at $t^{(LC)}_i$.

Next, we give the definition of an embedded chain. 

\begin{definition}[Embedded Chain]
An embedded chain $\hat{q}(s)$ contains a subset of a process $q(t)$, and relabels the selected elements of the original chain into a consecutive time series. Specifically, $\{\hat{q}(s)|s=0,1,\ldots\} = \{q(t)|t\in \mathbb{T}\}$ for some index set $\mathbb{T}\subset \mathbb{N}$. 
\end{definition}

\begin{figure*}[bh]
\centering
\includegraphics[width=0.8\linewidth]{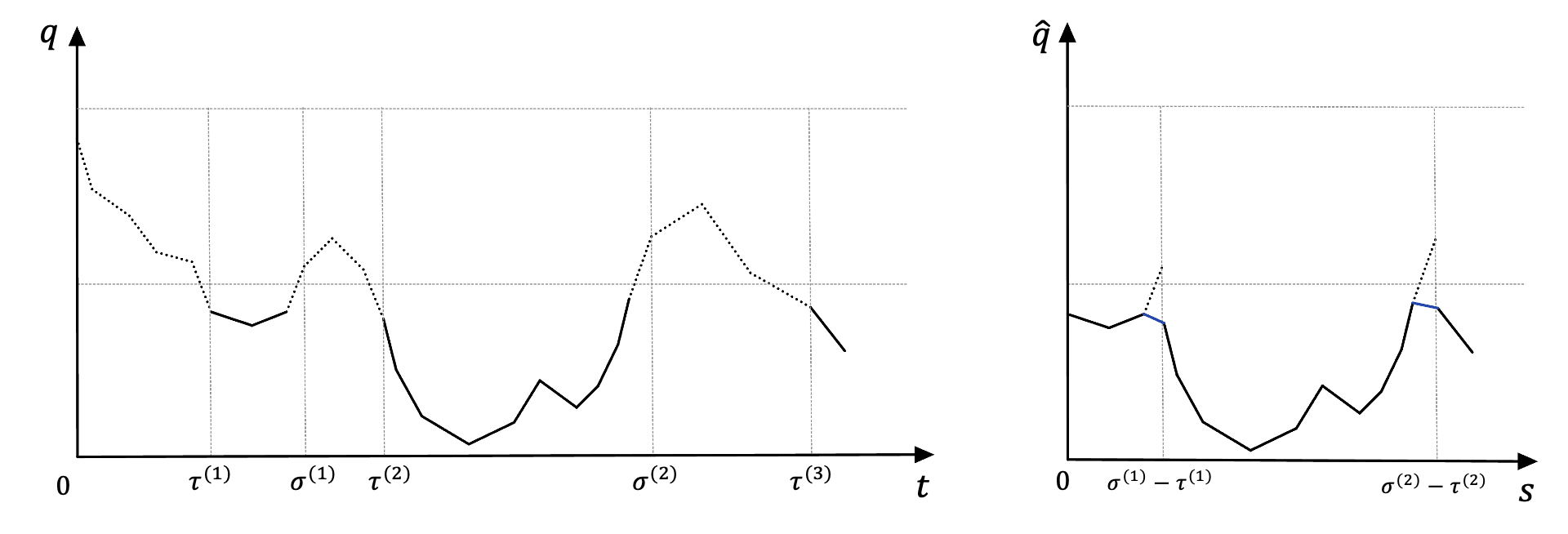}
\caption{An illustration of an embedded chain. }
\label{fig:embedded1}
\end{figure*}

See Figure~\ref{fig:embedded1} as an illustration of an embedded chain, where the left is the original process and the right is an embedded chain consisting of all points with values lower than a certain constant (depicted by the middle horizontal dashed line). This embedded chain is relabeled as a process with time series starting from 0.

Define stopping times $\tau^{(0)}=0,\sigma^{(0)}=0$, and for $k\geq 1, k\in\mathbb{Z}$,
\[\tau^{(k)}=\min\Big\{t\geq \tau^{(k)}:q_{a^*}(t)< \frac 1 2, t\notin S_{LC}\Big\},\]
\[\sigma^{(k)}=\min\Big\{t\geq \sigma^{(k-1)}:q_{a^*}(t)\geq \frac 1 2, t\notin S_{LC}\Big\}.\]

As in the left figure of Figure~\ref{fig:embedded2}, we let 
\[\tau_s=\min\Big\{t>\sigma_s:q_{a^*}(t+t_s)<\frac 1 2, t\notin S_{LC}\Big\},\]
\[\sigma_s=\min\Big\{t>0:q_{a^*}(t+t_s)\geq \frac 1 2, t\notin S_{LC}\Big\}. \]

Specifically, we define $\tau^{(k)}=\infty$ if there is no such $t$ exists for $k$, similarly for $\sigma^{(k)}=\infty$. 

Then, we prove the main Theorem~\ref{thm:main} considering the following cases.

\subsection{When $p_{a^*}(t)\geq 1/2$}
When $p_{a^*}(t)\geq 1/2$, i.e., $q_{a^*}(t)\leq 1/2$, we divide the process into two cases. 

\paragraph{Case 1.} $p_{a^*}(t)\geq 1/2$ and $t\notin S_{LC}$.

\begin{lemma}\label{thm:smallCorr1}
    Define the process $\hat q(s):s\in \mathbb{Z}^+$ to be the embedded chain from $q_{a^*}(t)$ with $t$ satisfying $q_{a^*}(t)\leq 1/2$ and $t\notin S_{LC}$, then 
    \begin{enumerate}
        \item The process $\hat q(s)$ is a positive supermartingale satisfying $\mathbb{E}[\hat q(s+1)|\hat{H}(s)]-\hat{q}(s)\leq -\frac{\alpha\Delta}{2K}\hat{q}(s)^2$.
        \item With probability one, $\hat q(s)\rightarrow 0$ as $s\rightarrow +\infty$. 
        \item $\mathbb{E}[\hat q(s)]\leq \frac{2K}{4K+\alpha s \Delta}$. 
    \end{enumerate}
\end{lemma}

Lemma~\ref{thm:smallCorr1} states the convergence of an embedded chain $\{\hat q(s)\}$ from the original process $\{q_{a^*}(t)\}$ and proves its convergence rate. The embedded chain includes the time steps with $q_{a^*}(t)\leq 1/2$ and $t\notin S_{LC}$, which is shown in the right figure of Figure~\ref{fig:embedded2}. 
Here, we examine only the steps satisfying $t\notin S_{LC}$, i.e., the steps not in the recovery process of any large corruptions $c(t_{LC})$.

\begin{proof}
    
    \emph{Property 1.} If there is a small corruption $c(t_s)(\leq \Delta/4)$ at step $t_s$, by the definition of the embedded chain, $\hat q(s)=q_{a^*}(t_s)$ as well as the fact that 
    $$ 
    \hat q(s+1)\leq \left\{
    \begin{aligned}
    q_{a^*}(t_s+1) & \ \ \text{ if }\  q_{a^*}(t_s+1)<\frac 1 2,\\
    q_{a^*}(t_s+\tau_s) & \ \ \text{ if }\ q_{a^*}(t_s+1)\geq\frac 1 2.
    \end{aligned}
    \right.
    $$
    Here, the inequality is because, for the removed recovery process $(t_c,t_c']$ of a large corruption, we must have $q_{a^*}(t_c')\leq q_{a^*}(t_c)$.
    
    Since $q_{a^*}(t_s+\tau_s)<\frac 1 2$, we know $\hat q (s+1)\leq q_{a^*}(t_s+1)$. 

    From the fact that the optimal arm $a^*$ is the leading arm at time $t_s$, we have that $\forall a\neq a^*$, 
    \[p_a(t_s+1)=p_a(t_s)+\alpha p_a(t_s)^2\Big[\frac{I_a(t_s) R_a(t_s)}{p_a(t_s)}-\frac{I_{a^*}(t_s) R_{a^*}(t_s)}{p_{a^*}(t_s)}\Big],\]
    where $ R$ is the corrupted reward. Thus, in expectation, 
    \begin{align*}
    &\mathbb{E}[p_a(t_s+1)-p_a(t_s)|H(t_s)]\\
    \leq& \alpha p_a(t_s)^2\big((r_a+\Delta/4)-(r_{a^*}-\Delta/4)\big)\\
    \leq& -\frac{\alpha\Delta}{2} p_a(t_s)^2.
    \end{align*}
    Because $q_{a^*}=1-p_{a^*}=\sum_{a:a\neq a^*}p_a$, 
    \[\mathbb{E}[q_{a^*}(t_s+1)-q_{a^*}(t_s)|H(t_s)]\leq -\frac{\Delta}{2}\sum_{a:a\neq a^*}\alpha p_a(t_s)^2.\]
    Thus, 
    \begin{align}
        \mathbb{E}[\hat q(s+1)|H(t_s)]-\hat q(s)&\leq \mathbb{E}[\hat q(s+1)-\hat q(s)|H(t_s)]\notag\\
        &\leq -\frac{\Delta}{2}\sum_{a:a\neq a^*}\alpha p_a(t_s)^2\notag\\
        &\leq -\frac{\alpha \Delta}{2K}\hat q(s)^2 \label{eq:super1}
    \end{align}
    where the last inequality holds because
    \begin{eqnarray}
        \sum_{a:a\neq a^*}p_a(t_s)^2&=& (K-1)\sum_{a:a\neq a^*}\frac{p_a(t_s)^2}{K-1}\\
        &\geq& (n-1)\Big(\sum_{a:a\neq a^*}\frac{p_a(t_s)}{K-1}\Big)^2\\
        &=& \frac{q_{a^*}(t_s)^2}{K-1}\\
        &\geq& \frac{\hat q(s)^2}{K}.\label{eq:jensen1}
    \end{eqnarray}
    Therefore, $\hat q(s)$ is a supermartingale. 

    \emph{Property 2.} Because $\hat q(s)$ is positive, from Doob's supermartigale convergence theorem, the limit $\lim_{s\rightarrow \infty}\hat q(s)$ exists. Next, we prove this limit is zero. It is sufficient that we prove $\lim\inf_{s\rightarrow \infty}\hat q(s)=0$. 
    
    We let $\phi_m=\min\{s\geq 1:\hat q(s)<1/m\}$. We first show that $\phi_m<\infty$. 
    From equation~(\ref{eq:super1}), the following holds when $\hat q(s)\geq \frac 1 m$, 
    \begin{equation}
        \mathbb{E}[\hat q(s+1)|H(t_s)]-\hat q(s)\leq-\frac{\alpha \Delta}{2K}\hat q(s)^2\leq -\frac{\alpha \Delta}{2K}\frac{1}{m^2}.  
    \end{equation}
    From optional stopping theorem, 
    \[\mathbb{E}[\hat q(\phi_m\wedge s)]+\frac{\alpha \Delta}{2K}\frac{1}{m^2}\mathbb{E}[\phi_m\wedge s]\leq \mathbb{E}[\hat q(0)].\]
    Then, applying the monotone converge theorem, we get, 
    \[\mathbb{E}[\phi_m]\leq \lim_{s\rightarrow \infty}\mathbb{E}[\phi_m\wedge s]\leq \frac{2Km^2}{2\alpha\Delta}\mathbb{E}[\hat q(0)]<\infty. \]
    Therefore, $\phi_m<\infty$ with probability 1. Then, we can define a sequence of stopping times $\psi_m=\min\{s\geq \psi_{m-1}:\hat q(s)<\frac 1 m\}$, each of which is finite w.p. 1 and $\hat q(\psi_m)\rightarrow 0$, which gives $\lim\inf_{s\rightarrow \infty}\hat q(s)=0$. 
    Therefore, $\lim_{s\rightarrow \infty}\hat q(s)=0$. 
    
    \emph{Property 3.} Taking expectation on both sides of equation~(\ref{eq:super1}), applying Lemma~\ref{lemma:bound1} (as follows) and using the fact that $\mathbb{E}[\hat q(0)]\leq\frac 1 2$, we get \[\mathbb{E}[\hat q(s)]\leq \frac{\mathbb{E}[\hat q(0)]}{1+\frac{\alpha\Delta}{2K}\mathbb{E}[\hat q(0)]s}\leq \frac{2K}{4K+\alpha\Delta s}. \]
    
\end{proof}
\begin{lemma}\label{lemma:bound1}
If $(a_t:t\in \mathbb{N})$ is a sequence of positive real numbers satisfying 
\[a_{t+1}\leq a_t-\gamma a_t^2\label{eq:lammabound1-1}\]
for some $\gamma>0$, then $\forall T\in\mathbb{N}$, 
\[a_T\leq \frac{a_0}{1+\gamma T a_0}.\]
\end{lemma}

\begin{proof}
    From inequality~(\ref{eq:lammabound1-1}), we know that
    \[\frac{a_{t+1}-a_t}{a_t^2}\leq -\gamma. \]
    Because $0<a_{t+1}<a_t$, 
    \[\frac{1}{a_t}-\frac{1}{a_{t+1}}=\frac{a_{t+1}-a_t}{a_ta_{t+1}}\leq \frac{a_{t+1}-a_t}{a_t^2}\leq -\gamma.\]
    Telescoping the previous inequality for $t=0,\ldots T$, we get 
    \[\frac{1}{a_0}-\frac{1}{a_T}\leq -\gamma T\Rightarrow a_T\leq \frac{a_0}{1+\gamma T a_0}.\]
\end{proof}

\paragraph{Case 2.} $p_{a^*}(t)\geq 1/2$ and $t\in S_{LC}$. 

\begin{lemma}\label{thm:largeCorr1}
    The total regret $Rg_{LC}$ occurred during the recovery process of large corruptions where $q_{a^*}(t)\leq \frac{1}{2}$ can be upper bounded as
    \begin{equation*}
        Rg_{LC} \le {\frac{4C}{\Delta}}.
    \end{equation*}
\end{lemma}

\begin{proof}

Note that for any recovery process $(t_s,t_s']$, if $\exists t\in (t_s,t_s']$ satisfies $q_{a^*}(t)\leq 1/2$, then $q_{a^*}(t_s)<1/2$. 

\begin{figure*}[htb]
\centering
\includegraphics[width=0.8\linewidth]{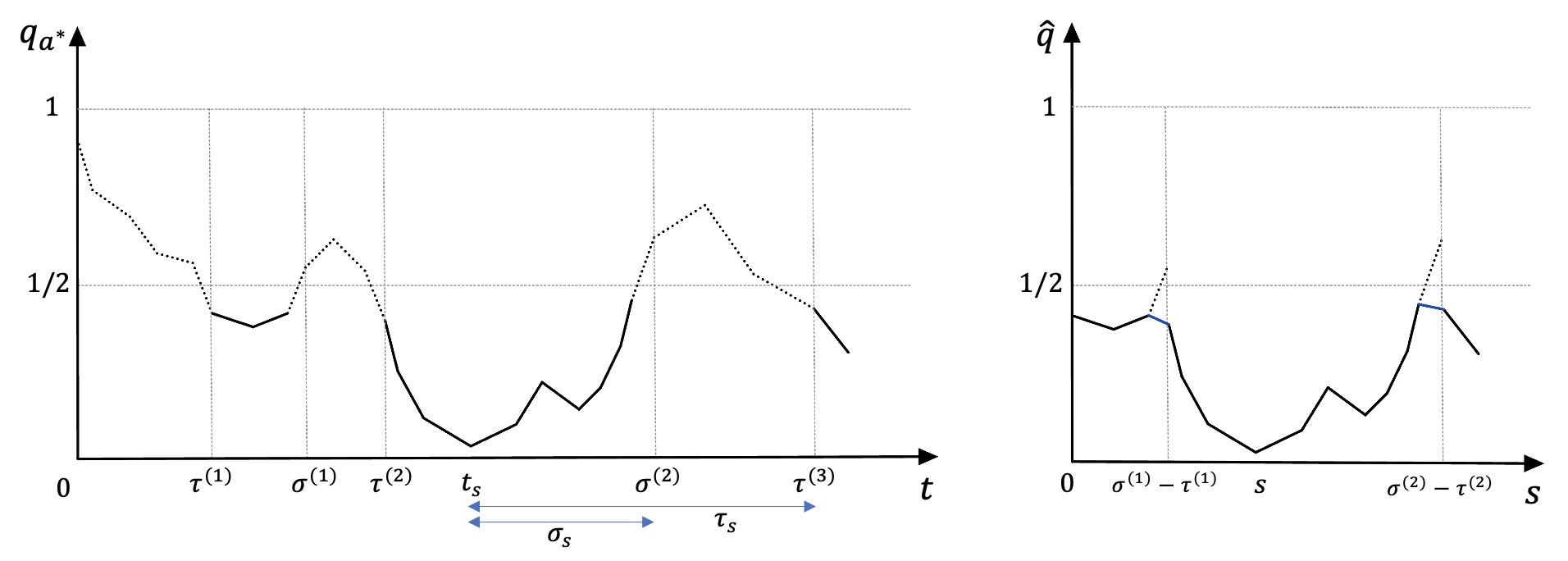}
\caption{An illustration of $\hat q(s)$ for Lemma~\ref{thm:smallCorr1}}
\label{fig:embedded2}
\end{figure*}

Consider the time step $t_s$ that satisfies $q_{a^*}(t_s)\leq 1/2$ and $c(t_s) > \Delta/4$. 
Then, for such $t_s$, it holds that 
\begin{align*}
\mathbb{E}[p_a(t_s+1)-p_a(t_s)|H(t_s)]&\leq \alpha p_a(t_s)^2\Big((r_a+c(t_s))-(r_{a^*}-c(t_s))\Big)\\
&\leq \alpha(2c(t_s)-\Delta) p_a(t_s)^2.
\end{align*}
Because $q_{a^*}=1-p_{a^*}=\sum_{a:a\neq a^*}p_a$, 
\begin{align*}
\mathbb{E}[q_{a^*}(t_s+1)|H(t_s)]-q_{a^*}(t_s)\leq& (2c(t_s)-\Delta)\sum_{a:a\neq a^*}\alpha p_a(t_s)^2\\
\leq& (2c(t_s)-\Delta)\frac{\alpha}{K}q_{a^*}(t_s)^2,
\end{align*}
where we used equation~(\ref{eq:jensen1}) in the last inequality. 

From equation~(\ref{eq:super1}), 
\begin{align*}
\mathbb{E}[q_{a^*}(t+1)|H(t)]-q_{a^*}(t)&=\mathbb{E}[q_{a^*}(t+1)-q_{a^*}(t)|H(t)]\leq -\frac{\alpha \Delta}{2K}q_{a^*}(t)^2
\end{align*}
for $t$ satisfying $q_{a^*}(t)\leq 1/2$ and $c(t)\leq \Delta/4$, including those $t\in S_{LC}$ but not including those $t\in\{t_i^{(LC)}, \forall i\}$.

If there is a large corruption in only one step, say step $t_0$ with corruption level $c(t_0)$, then $\mathbb{E}[q_{a^*}]$ will increase after $t_0$ and subsequently gradually decrease, as shown in Figure~\ref{fig:embedded3}. 
We want to upper bound the expected number of steps during the recovery process after corruption $c(t_0)$. 
We use the optional stopping theorem to give such a bound. 
Let $\phi=\min\{t> t_0:q_{a^*}(t)\leq q_{a^*}(t_0)\}$. When $t_0<t\leq \phi$, we have $q_{a^*}(t)\geq q_{a^*}(t_0)$. Then, 
\begin{equation}
    \mathbb{E}[q_{a^*}(t+1)|H(t)]-q_{a^*}(t)\leq-\frac{\alpha \Delta}{2K}q_{a^*}(t)^2\leq -\frac{\alpha \Delta}{2K}q_{a^*}(t_0)^2.  
\end{equation}
From optional stopping theorem, 
\begin{align*}
&\ \mathbb{E}[q_{a^*}(\phi\wedge t)]+\frac{\alpha \Delta}{2K}q_{a^*}(t_0)^2\mathbb{E}[\phi\wedge t]\\
\leq&\ \mathbb{E}[q_{a^*}(\phi\wedge (t_0+1))]+\frac{\alpha \Delta}{2K}q_{a^*}(t_0)^2\mathbb{E}[\phi\wedge (t_0+1)].
\end{align*}
Then, applying the monotone converge theorem, we get, 
\begin{align*}
\mathbb{E}[\phi-t_0-1]&\leq \lim_{t\rightarrow \infty}\mathbb{E}[\phi\wedge t]-\mathbb{E}[\phi\wedge (t_0+1)]\\
&\leq \frac{2K}{\alpha\Delta q_{a^*}(t_0)^2}(\mathbb{E}[q_{a^*}(t_0+1)]-\mathbb{E}[q_{a^*}(\phi)]). 
\end{align*}
Thus,
\begin{align*}
\mathbb{E}[\phi-t_0]\leq& \frac{2K}{\alpha\Delta q_{a^*}(t_0)^2}\Big((2c(t)-\Delta)\frac{\alpha}{K}q_{a^*}(t_0)^2+\frac{\alpha \Delta}{2K}q_{a^*}(t_0)^2\Big)+1=\frac{4c(t)}{\Delta}. 
\end{align*}

\begin{figure}
\centering
\includegraphics[width=0.7\linewidth]{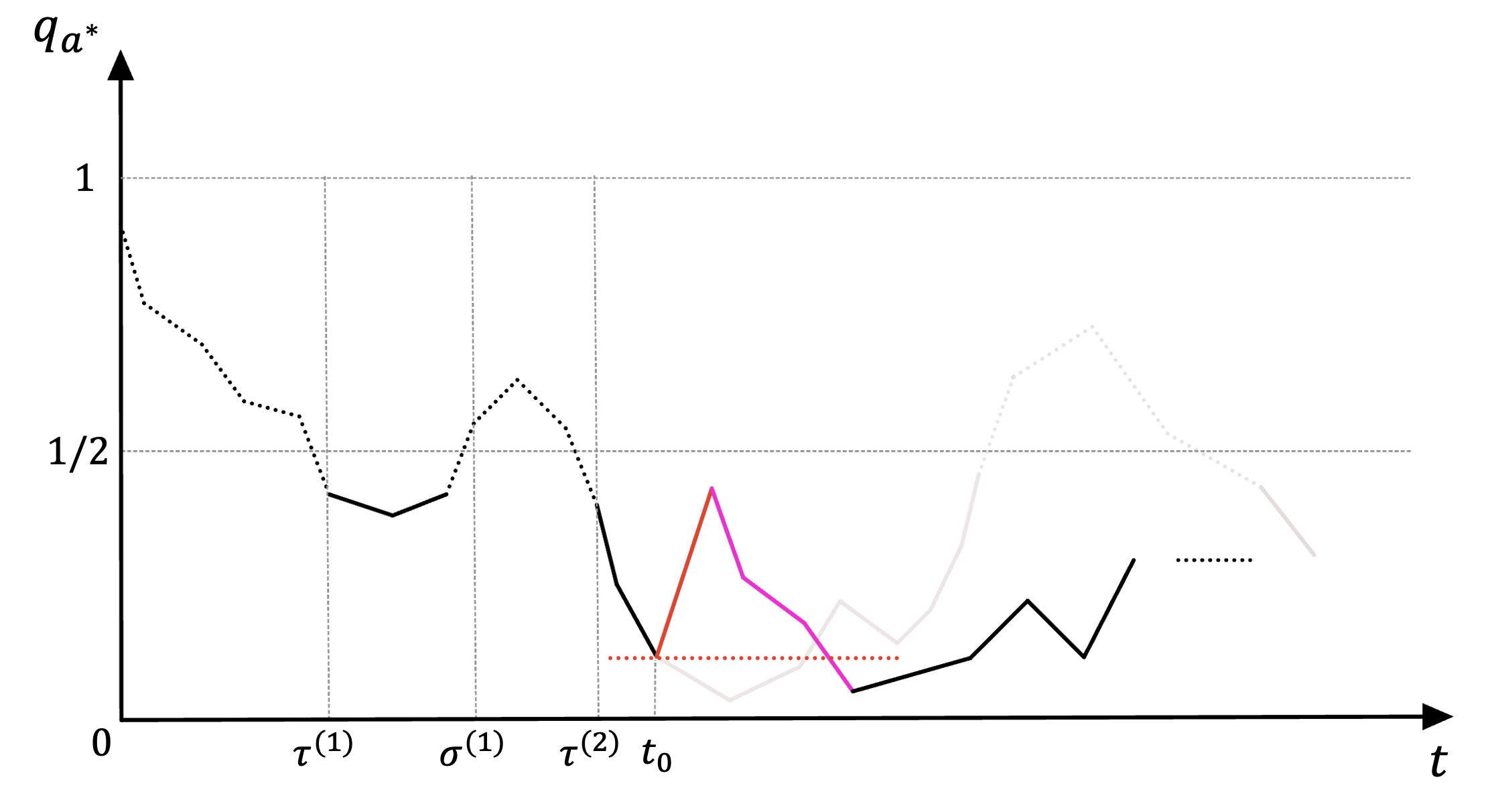}
\caption{An illustration of a single corruption for Lemma~\ref{thm:largeCorr1}}
\label{fig:embedded3}
\end{figure}

% \jiayuan{Explain Figure 4 and consecutive corruption analysis. }
If there are consecutive corruptions (other corruptions come before recovery from the previous corruption), then the total extra regret incurred by these corruptions is upper bounded by the regret calculated by considering these corruption steps separately. 

This can be shown from inner corruptions to outer corruptions. We first prove the case depicted in Figure~\ref{fig:consecutive1} and then extend to general corruption patterns. In Figure~\ref{fig:consecutive1}, $t_i$'s are the steps with $c(t_i)>\Delta/4$ and $t_i'$ is the first time step after recovering from corruption $c(t_i)$, $\forall i$.  

\begin{figure}
\centering
\includegraphics[width=0.4\linewidth]{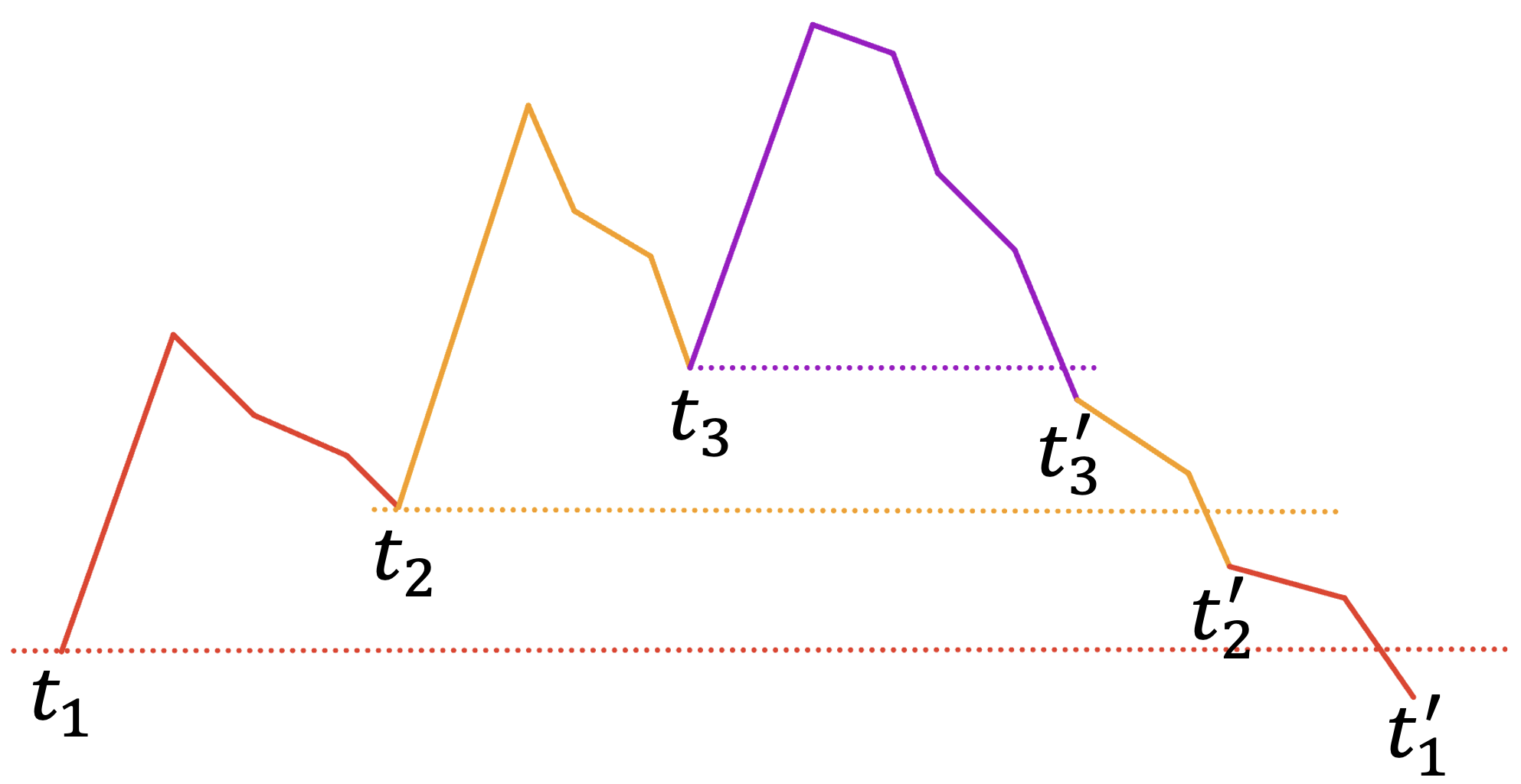}
\caption{An illustration of consecutive corruptions for Lemma~\ref{thm:largeCorr1}}
\label{fig:consecutive1}
\end{figure}

First, the recovery process for the corruption at time $t_3$  (colored purple) lasts for $\phi^{(3)}:=t_3'-t_3$ (considering the time interval $(t_3,t_3']$), which can be bounded the same as previously discussed, 
\[\mathbb{E}[\phi^{(3)}]=\mathbb{E}[t_3'-t_3]\leq\frac{4c(t_3)}{\Delta}. \] 

Next, consider the recovery process for corruption at time $t_2$ (colored yellow), which lasts for $(t_2'-t_3')+(t_3-t_2)$. 
Formally, we consider the time intervals $[t_2,t_3]$ and $(t_3', t_2']$ as a whole and apply the optional stopping theorem to it. 
We consider an embedded chain $\{\hat q^{(2)}(s)|s=0,1,2,\ldots, M-t_3'+t_3-t_2\}$ of the process $\{q_{a^*}(t)\}$ with $t\in[t_2,t_3]\cup [t_3'+1, M)$, where $M$ is a large number no smaller than the recovery time step $t_2'$. 

Let $\phi^{(2)}=\min\{s>0:\hat q^{(2)}(s)\leq \hat q^{(2)}(0)\}$. When $0<s\leq \phi^{(2)}$, we have $\hat q^{(2)}(s)\geq \hat q^{(2)}(0)$. 

Then, for $s\in[1,t_3-t_2-1]\cup [t_3-t_2+1,M-t_3'+t_3-t_2]$, corresponding to $t\in[t_2+1,t_3-1]\cup [t_3'+1, M]$ in process$\{q_{a^*}(t)\}$, it satisfies that
\begin{equation}\label{eq:optional3}
    \mathbb{E}[\hat q^{(2)}(s+1)|H(s)]-\hat q^{(2)}(s)\leq-\frac{\alpha \Delta}{2K}\hat q^{(2)}(s)^2\leq -\frac{\alpha \Delta}{2K}\hat q^{(2)}(0)^2. 
\end{equation}
For $s=t_3-t_2-1$,  
\begin{eqnarray*}
    &&\mathbb{E}[\hat q^{(2)}(s+1)|H(s)]-\hat q^{(2)}(s)\\
    &=&\mathbb{E}[\hat q^{(2)}(t_3-t_2)|H(t_3-t_2-1)]-\hat q^{(2)}(t_3-t_2-1) \\
    &=&\mathbb{E}[q_{a^*}(t_3'+1)|H(s)]-q_{a^*}(t_3)\\
    &\leq& \mathbb{E}[q_{a^*}(t_3'+1)|H(s)]-q_{a^*}(t_3')\\
    &\leq& -\frac{\alpha \Delta}{2K}q_{a^*}(t_3')^2\\
    &\leq& -\frac{\alpha \Delta}{2K}q_{a^*}(t_2)^2\\
    &=& -\frac{\alpha \Delta}{2K}\hat q^{(2)}(0)^2. 
    % \leq-\frac{\alpha \Delta}{2K}\hat q^{(2)}(s)^2\leq -\frac{\alpha \Delta}{2K}\hat q^{(2)}(0)^2. 
\end{eqnarray*}
Therefore, $\mathbb{E}[\hat q^{(2)}(s+1)|H(s)]-\hat q^{(2)}(s)\leq -\frac{\alpha \Delta}{2K}\hat q^{(2)}(0)^2$ holds for all $s\in[1,M-t_3'+t_3-t_2]$, and therefore the process $\{\phi^{(2)}(s)|s\geq 1\}$ is a supermartingale.  

From optional stopping theorem, 
\begin{eqnarray*}
&&\mathbb{E}[\hat q^{(2)}(\phi^{(2)}\wedge s)]+\frac{\alpha \Delta}{2K}\hat q^{(2)}(0)^2\mathbb{E}[\phi^{(2)}\wedge s]\\
&\leq& \mathbb{E}[\hat q^{(2)}(\phi^{(2)}\wedge 1)]+\frac{\alpha \Delta}{2K}\hat q^{(2)}(0)^2\mathbb{E}[\phi^{(2)}\wedge 1].
\end{eqnarray*}
Then, applying the monotone converge theorem, we get, 
\begin{eqnarray*}
\mathbb{E}[\phi^{(2)}-1]&\leq& \lim_{s\rightarrow \infty}\mathbb{E}[\phi^{(2)}\wedge s]-\mathbb{E}[\phi^{(2)}\wedge 1]\\
&\leq& \frac{2K}{\alpha\Delta \hat q^{(2)}(0)^2}(\mathbb{E}[\hat q^{(2)}(1)]-\mathbb{E}[\hat q^{(2)}(\phi^{(2)})]). 
\end{eqnarray*}
Thus, % \jiayuan{The first term in the parenthesis of the next inequality}
\begin{align*}
\mathbb{E}[\phi^{(2)}]&\leq \frac{2K}{\alpha\Delta \hat q^{(2)}(0)^2}\Big((2c(t_2)-\Delta)\frac{\alpha}{K}\hat q^{(2)}(0)^2+\frac{\alpha \Delta}{2K}\hat q^{(2)}(0)^2\Big)+1\\
&=\frac{4c(t_2)}{\Delta}. 
\end{align*}

The same analysis holds for the recovery process for corruption at time $t_1$ (colored red), where the expected recovery time for $c(t_1)$ can be upper bounded by 
\[\mathbb{E}[\phi^{(1)}]\leq \frac{4c(t_1)}{\Delta}.\]

Put them all together and we can get that the expected number of total recovery steps needed can be upper bounded by $\sum_{i}\frac{4c(t_i)}{\Delta}=\frac{4C}{\Delta}$ where $C$ is the total corruption level. 

If there are multiple consecutive corruptions at time steps $t_1, t_2, \ldots$, we can always dissect the process and combine the recovery process for each corruption $t_i$ together into an embedded chain where equation~(\ref{eq:optional3}) is satisfied. Then, we can upper bound the number of steps in such an embedded chain by $\frac{4c(t_i)}{\Delta}$ using the optional stopping theorem. 

\begin{figure}[htb]
\centering
\includegraphics[width=0.4\linewidth]{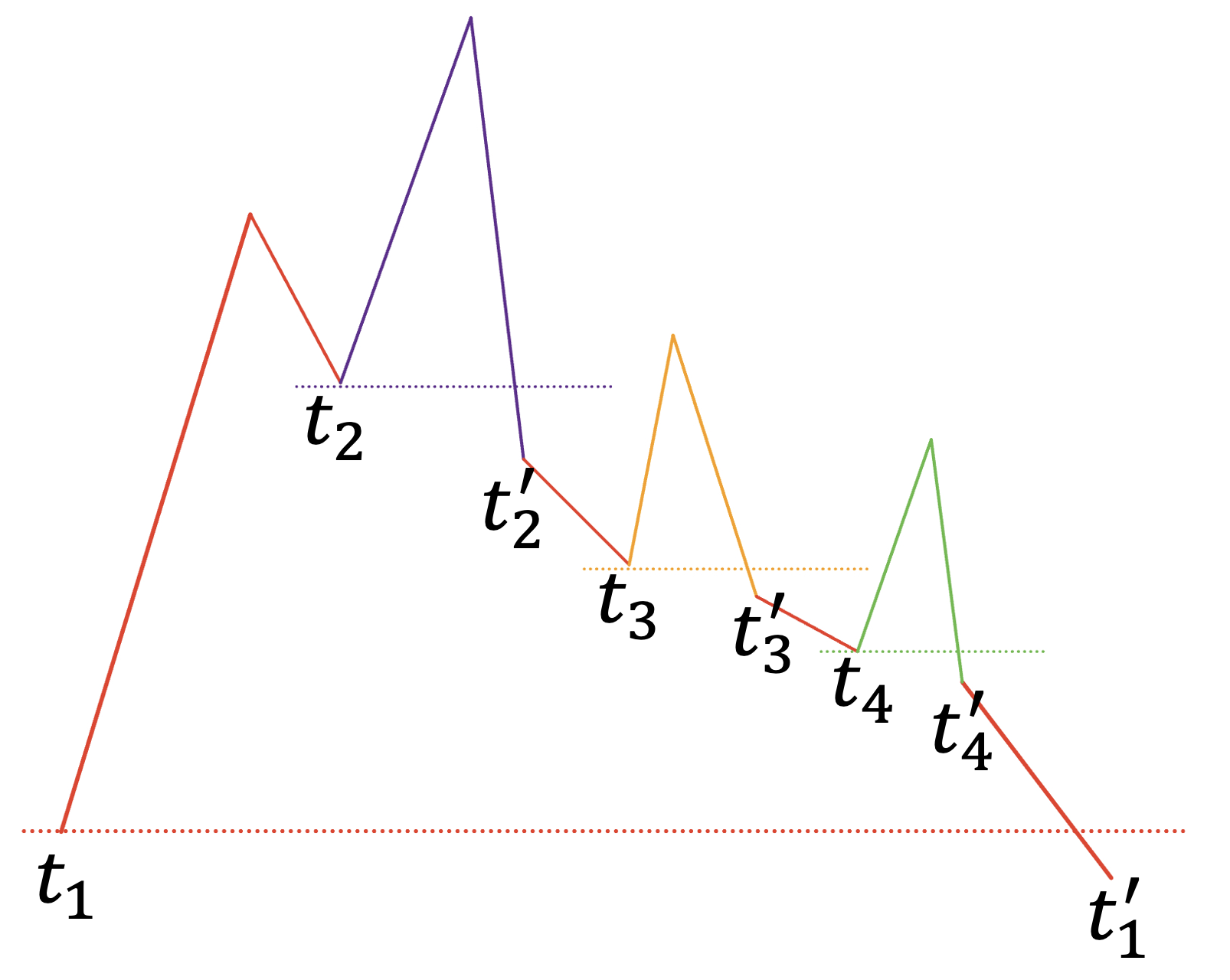}
\caption{An illustration of general consecutive corruptions for Lemma~\ref{thm:largeCorr1}}
\label{fig:consecutive2}
\end{figure}

A depiction of the general case can be seen in Figure~\ref{fig:consecutive2}, where several dissections $(t_1,t_2], (t_2',t_3], (t_3',t_4], (t_4,t_1']$ should be combined together into the corresponding embedded chain for corruption at $t_1$. 

This dissection-then-combination process can be operated from ``inner'' corruptions to ``outer'' corruptions. 
Specifically, we label the corruption steps as $t_i$'s and label $t_i'$ to be the first recovery time step. Next, these time steps are ordered according to the time axis, e.g., $t_1,t_2,t_2',t_3,t_3',t_4,t_4',t_1'$ is the ordering for the example in Figure~\ref{fig:consecutive2}. Then, we can iteratively find two consecutive $(t_i,t_i')$ in the ordering, upper bound the number of recovery steps between $t_i$ and $t_i'$ by $\frac{4c(t_i)}{\Delta}$ using the optional stopping theorem, remove all time steps in the range $(t_i, t_i']$, combine the rest $q_{a^*}$ values in time intervals $[0,t_i]\cup (t_i',\infty)$ into an updated process, remove the two time steps $t_i, t_i'$ in the ordering. 

As a result, the expected number of total recovery steps needed can be upper bounded by $\sum_{i}\frac{4c(t_i)}{\Delta}=\frac{4C}{\Delta}$ in general case, where each such step gives at most 1 regret. Therefore, $Rg_{LC}$ is upper bounded by $\frac{4C}{\Delta}\times 1=\frac{4C}{\Delta}$. 
\end{proof}

\subsection{When $p_{a^*}(t)< 1/2$}

When $p_{a^*}(t) < 1/2$, we also divide the process into two cases. 

\paragraph{Case 1.} $p_{a^*}(t) < 1/2$ and $t\in S_{LC}$. 

\begin{lemma}\label{lemma:largeCorruptionsmallp}
The expected number of steps during the recovery processes of all large corruptions (i.e., $t\in S_{LC}$) where $q_{a^*}(t)>1/2$ is upper bounded by $O({\frac{C}{\Delta}})$. 
\end{lemma}

\begin{proof}

Note that we have assumed $0<\alpha<\frac{\Delta}{r^*-\Delta}$. Thus, 
\[1+\alpha<\frac{r^*}{r^*-\Delta}.\]
As a result, there exists a constant $\epsilon>0$ such that $(1+\epsilon)(1+\alpha)<\frac{r^*}{r^*-\Delta}$, which implies that there exists another constant $\xi$ such that 
\[\xi:=\alpha \frac{r^*}{1+\alpha}-\alpha (r^*-\Delta)(1+\epsilon)>0. \]

When optimal arm $a^*$ is not the leading arm, during which 
$$ 
p_{a^*}(t+1)=\left\{
\begin{aligned}
&(1+\alpha)p_{a^*}(t) && \ \ \text{ w.p. }\  p_{a^*}(t)r^*\\
&p_{a^*}(t)-\alpha p_{a^*}^2(t)\frac{1}{p_{a_l}(t)} && \ \ \text{ w.p. }\ p_{a_l}(t)r_{a_l}\\
&p_{a^*}(t) && \ \ \text{ otherwise}
\end{aligned}
\right.
$$

If we denote $x(t)=p_{a^*}^{-1}(t)$, then 
$$ 
x(t+1)=\left\{
\begin{aligned}
&x(t)-\frac{\alpha}{1+\alpha}x(t) && \ \ \text{ w.p. }\  r^*/x(t)\\
&x(t)+\alpha\frac{x(t)}{p_{a_l}(t)x(t)-\alpha} && \ \ \text{ w.p. }\ p_{a_l}(t)r_{a_l}\\
&x(t) && \ \ \text{ otherwise}
\end{aligned}
\right.
$$

Thus, when there is no corruption, 
\begin{align}
\mathbb{E}[x(t+1)|H(t)]-x(t)&=\alpha r_{a_l}(t)\frac{p_{a_l}(t)x(t)}{p_{a_l}(t)x(t)-\alpha}-\frac{\alpha r^*}{1+\alpha}\label{eq:smallpInduction2}\\
&\leq \alpha(r^*-\Delta)(1+\epsilon)-\frac{\alpha r^*}{1+\alpha}=-\xi,\notag
\end{align}
where the above equation holds because for the leading arm, $p_{a_l}(t)>1/K$ and $r_{a_l}\leq r^*-\Delta$. 

If there is corruption in round $t$, then only corruption on the optimal arm or the leading arm will change the update function.

The update rule becomes to
$$ 
x(t+1)=\left\{
\begin{aligned}
&x(t)-\frac{\alpha}{1+\alpha}x(t) && \ \ \text{ w.p. }\  r'_{a^*}/x(t)\\
&x(t)+\alpha\frac{x(t)}{p_{a_l}(t)x(t)-\alpha} && \ \ \text{ w.p. }\ r_{a_l}'(t)p_{a_l}(t)\\
&x(t) && \ \ \text{ otherwise}
\end{aligned}
\right.
$$
Then,
\begin{eqnarray}
&&\mathbb{E}[x(t+1)|H(t)]-x(t)\\
&\leq &\alpha (r_{a_l}(t)+c(t))\frac{p_{a_l}(t)x(t)}{p_{a_l}(t)x(t)-\alpha}-\frac{\alpha (r^*-c(t))}{1+\alpha}\notag\\
&\leq& \alpha(r^*-\Delta+c(t))(1+\epsilon)-\frac{\alpha (r^*-c(t))}{1+\alpha}\\
&=& -\xi+\alpha c(t)\big(1+\epsilon+\frac{1}{1+\alpha}\big).\label{eqn:corr1-3}
\end{eqnarray}

The right hand side of equation~(\ref{eqn:corr1-3}) can be greater than 0, which means that after the corruption time step, the expected value of $x(t)$ can increase.
Hence, the increment can be counterbalanced after the next $\big\lceil\alpha c(t)\big(1+\epsilon+\frac{1}{1+\alpha}\big)/\xi\big\rceil=\big\lceil c(t)/\zeta\big\rceil$ time steps that do not have corruptions, 
where we define the constant $\zeta := \frac{\xi}{\alpha\big(1+\epsilon+\frac{1}{1+\alpha}\big)}$.
Therefore, after a large corruption $c(t_s)$ at step $t_s$, if $a^*$ is not the leading arm, then the impact of $c(t)$ will be eliminated after $1+\frac{c(t)}{\zeta}$ steps in expectation.

On the other hand, if $a^*$ is still the leading arm after a large corruption $c(t_s)$, then
the impact of $c(t)$ will be eliminated after expected $\frac{4c(t)}{\Delta}$ steps (similar to the proof of Lemma~\ref{thm:largeCorr1}). 
Moreover, if some large corruption occurs in a recovery process $[t_s, t_s')$, we can still use the same technique as Lemma \ref{thm:largeCorr1} to deal with them.

Therefore, noting that there are at most $\lfloor 4C/\Delta\rfloor$ steps with large corruptions, the expected number of steps during these recovery processes with $p_{a^*}(t) < 1/2$ can be upper bounded by $\sum_{s}(1+\frac{c(t^{(LC)}_s)}{\zeta}+\frac{4c(t^{(LC)}_s)}{\Delta})\leq \frac{4C}{\Delta} + \frac{C}{\zeta}+\frac{C}{\Delta} = \frac{C}{\zeta}+\frac{5C}{\Delta}$.

Recall that 
\begin{align*}
    \zeta &= \frac{\alpha \frac{r^*}{1+\alpha}-\alpha (r^*-\Delta)(1+\epsilon)}{\alpha\big(1+\epsilon+\frac{1}{1+\alpha}\big)}=\frac{r^*-(r^*-\Delta)(1+\epsilon)(1+\alpha)}{1+(1+\epsilon)(1+\alpha)} = \Theta(\Delta),
\end{align*}
hence $\frac{C}{\zeta}+\frac{5C}{\Delta} = O({C\over \Delta})$.
\end{proof}

\paragraph{Case 2.} $p_{a^*}(t) < 1/2$ and $t\notin S_{LC}$. 

In this case (Lemma~\ref{lemma:initRegret} and Lemma~\ref{lemma:bound2}), we only consider the remaining steps after removing all the large corruptions (along with their recovery processes). Note that for any recovery process $(t_s,t_s']$, it satisfies $q_{a^*}(t_s')\leq q_{a^*}(t_s)$, and thus we can put the rest (after removal of the recovery processes of large corruptions) together into one whole process that satisfies the conditions required in the proof of Lemma~\ref{lemma:initRegret} and Lemma~\ref{lemma:bound2}.

\begin{lemma}\label{lemma:initRegret}
If there is no time step $t$ with large corruption ($c(t)> \Delta/4$), then the expected time steps for $q_{a^*}$ to decrease from $1 - \frac{1}{K}$ (the initial value) to $\frac{1}{2}$ can be upper bounded by $\frac{K}{\xi} + \frac{8K}{\alpha\Delta}$. 
\end{lemma}

\begin{proof}
    The proof of Lemma~\ref{lemma:initRegret} follows a similar method as the proof of Lemma~\ref{lemma:largeCorruptionsmallp}. 
    
    First, consider the time steps between $p_{a^*}(0)=1/K$ and $p_{a^*}\leq 1/2$ where the arm $a^*$ is not the leading arm. From the update function, it satisfies that $|p_{a^*}(t+1)-p_{a^*}(t)|\leq \alpha p_{a^*}(t), \forall t$. Thus, from inequality~(\ref{eq:smallpInduction2}), the expected number of such steps can be upper bounded by $\frac{(1/K)^{-1}-{((1-\alpha)/2)^{-1}}}{\xi}\leq \frac{K}{\xi}$. 

    Next, consider the time steps between $p_{a^*}(0)=1/K$ and $p_{a^*}\leq 1/2$ where the arm $a^*$ is the leading arm. From a similar analysis as in Lemma~\ref{thm:largeCorr1}, the expected number of such steps can be upper bounded by $\frac{2K}{\alpha\Delta (1/2)^2}(\frac{1+\alpha}{2}-\frac{1}{K})\leq \frac{8K}{\alpha\Delta}$. 
\end{proof}

\begin{lemma}\label{lemma:bound2}
    If there does not exist any time step $t$ with large corruption ($c(t)> \Delta/4$), then the following holds. 
    \begin{enumerate}
        \item If $q_{a^*}(0)\leq \frac 1 2$, then there exists a constant $\rho<1$ such that $\mathbb{P}(\sigma^{(1)}<\infty)<\rho$, and $\mathbb{P}(\sigma^{(k)}<\infty|\sigma^{(k-1)}<\infty)<\rho$. 
        \item $Q_0:=\sum_{t=0}^{\infty}\mathbb{P}\big(q_{a^*}(t)\geq \frac 1 2\big|q_{a^*}(0)\leq \frac 1 2\big)<\infty$.
        \item With probability 1, $q_{a^*}(t)\rightarrow 0$ as $t\rightarrow \infty$. 
    \end{enumerate}
\end{lemma}

The proof of Lemma~\ref{lemma:bound2} is almost the same as the proof of Proposition 3 in~\cite{denisov2020regret}, mainly because with small corruption, the update function follows the same format as with zero corruption; e.g., $\mathbb{E}[q_{a^*}]$ decreases after both the zero corruption steps and the small corruption steps. 

Note that the results in Lemma~\ref{lemma:bound2} is based on the assumption that $0 < q_{a^*}(0) \leq \frac{1}{2}$, while our initial state is $q_{a^*}(0) = \frac{1}{K}$. Thus, regret in this case is upper bounded by the sum of regret upper bounds in Lemma~\ref{lemma:initRegret} and Lemma~\ref{lemma:bound2}.

\subsection{Main proof of Theorem~\ref{thm:main}}
Here, we restate our main theorem.

\textbf{Theorem 2.}
    If constant $\alpha<\frac{\Delta}{r^*-\Delta}$, then the SAMBA algorithm for multi-armed bandits problem with corruption level $C$ ensures regret \[Rg(T)=O\Big(\frac{K}{\Delta}\log T +\frac C \Delta \Big).\]

\begin{proof}%[Proof of Theorem~\ref{thm:main}]
  We bound the regret $Rg(T)$ as follows. 
  \begin{align}
      Rg(T)&\leq \sum_{a:a\neq a^*}\mathbb{E}\Big[\sum_{t=0}^{T-1}p_a(t)\Big]\\
      &\leq \mathbb{E}\Big[\sum_{t=0}^{T-1}\sum_{a:a\neq a^*}p_a(t)\Big] \\ 
      &\leq \mathbb{E}\Big[\sum_{t=0}^{T-1}q_{a^*}(t)\Big] \label{eq:regretineq1}
  \end{align}
    where we used the fact $r^*-r_a\leq 1$. Next, we properly partition the last term of Inequality~(\ref{eq:regretineq1}) and each part is bounded separately in the previous stated lemmas. 

    % \jiayuan{Modify here. }
  \begin{align}
      &\mathbb{E}\Big[\sum_{t=0}^{T-1}q_{a^*}(t)\Big]\\
      =&\sum_{t=0}^{T-1} \mathbb{E}\Big[q_{a^*}(t)\mathbb{I}\Big[q_{a^*}(t)\geq \frac 1 2\Big]\Big]+ \mathbb{E}\Big[\sum_{t=0}^{T-1}q_{a^*}(t)\mathbb{I}\Big[q_{a^*}(t)<\frac 1 2\Big]\Big]\label{eq:main11}\\
      \leq&\ Q_0+\frac{K}{\xi} + \frac{8K}{\alpha\Delta}+O\Big({C\over \Delta}\Big)+\mathbb{E}\Big[\sum_{t=0}^{T-1}q_{a^*}(t)\mathbb{I}\Big[q_{a^*}(t)<\frac 1 2\Big]\Big]\label{eq:main12}\\
      \leq&\ Q_0+\frac{K}{\xi} + \frac{8K}{\alpha\Delta}+O\Big({C\over \Delta}\Big)+\mathbb{E}\Big[\sum_{t=0}^{T-1}\hat q(s)\Big]+Rg_{LC}\label{eq:main13}\\
      \leq&\ Q_0+\frac{K}{\xi} + \frac{8K}{\alpha\Delta}+O\Big({C\over \Delta}\Big)+\sum_{s=0}^{T-1}\frac{2K}{4K+\alpha\Delta s}+\frac{4C}{\Delta}\label{eq:main14}\\
      \leq&\ Q_0+\frac{K}{\xi} + \frac{8K}{\alpha\Delta}+O\Big({C\over \Delta}\Big) +\frac{2K}{\alpha\Delta}\log T \label{eq:main15}\\
      \leq&\ O\Big(\frac{K}{\Delta}\log T +\frac C \Delta \Big)
      \label{eq:regretineq2}
  \end{align}
    Inequality~(\ref{eq:main12}) comes from Lemma~\ref{lemma:largeCorruptionsmallp}, Lemma~\ref{lemma:initRegret}, and Lemma~\ref{lemma:bound2}. 
    Inequality~(\ref{eq:main13}) further separates the regret (the last term of (\ref{eq:main12})) into two parts, one is from large corruption as described in Lemma~\ref{thm:largeCorr1} and another part is from small corruption as described in Lemma~\ref{thm:smallCorr1}; 
    Inequality~(\ref{eq:main14}) comes from Lemma~\ref{thm:smallCorr1}.  
\end{proof}

\end{document}